\documentclass{article}
\PassOptionsToPackage{numbers, compress}{natbib}
\usepackage{algorithm,algorithmicx}
\usepackage{algorithmicx}
\usepackage{algorithm} 
\usepackage{algpseudocode}
\usepackage{amsmath}
\usepackage{amssymb}
\usepackage{amsthm}
\usepackage{enumitem}
\DeclareMathOperator*{\argmax}{arg\,max}

\DeclareMathOperator{\E}{\mathbb{E}}
\newtheorem{theorem}{Theorem}
\usepackage{xcolor}

\usepackage{graphicx}
\usepackage{mathrsfs}
\usepackage{float}
\usepackage{algorithm,algpseudocode}
\usepackage[preprint]{neurips_2020}

\usepackage[utf8]{inputenc} 
\usepackage[T1]{fontenc}    
\usepackage{hyperref}       
\usepackage{url}            
\usepackage{booktabs}       
\usepackage{amsfonts}       
\usepackage{nicefrac}       
\usepackage{microtype}      
\usepackage{caption}
\title{GRAC: Self-Guided and Self-Regularized Actor-Critic}

\author{%
  Lin Shao$^{1}$, Yifan You$^{2}$, Mengyuan Yan$^{1}$, Qingyun Sun$^{3}$, Jeannette Bohg$^{1}$\\
  $^{1}$Stanford AI Lab, $^{3}$Department of Mathematics\\
  Stanford University\\
 \texttt{\{lins2, mengyuan, qysun, bohg\}@stanford.edu} \\
  $^{2}$Department of Computer Science\\
  University of California, Los Angeles\\
 \texttt{harry473417@ucla.edu} \\
\thanks{Toyota Research Institute ("TRI") provided funds to assist the authors with their research but this article solely reflects the opinions and conclusions of its authors and not TRI or any other Toyota entity. This article is also supported by Siemens and the HAI-AWS Cloud Credits.} 
}

\begin{document}

\maketitle

\begin{abstract}
Deep reinforcement learning~(DRL) algorithms have successfully been demonstrated on a range of challenging decision making and control tasks. One dominant component of recent deep reinforcement learning algorithms is the target network which mitigates the divergence when learning the Q function. However, target networks can slow down the learning process due to delayed function updates. Our main contribution in this work is a self-regularized TD-learning method to address divergence without requiring a target network. Additionally, we propose a self-guided policy improvement method by combining policy-gradient with zero-order optimization to search for actions associated with higher Q-values in a broad neighborhood. This makes learning more robust to local noise in the Q function approximation and guides the updates of our actor network. Taken together, these components define \emph{GRAC}, a novel self-guided and self-regularized actor critic algorithm. We evaluate \emph{GRAC} on the suite of OpenAI gym tasks, achieving or outperforming state of the art in every environment tested.

\end{abstract}

\section{Introduction}
\vspace{-5pt}
Reinforcement learning~(RL) studies decision-making with the goal of maximizing total discounted reward when interacting with an environment. Leveraging high-capacity function approximators such as neural networks, Deep reinforcement learning~(DRL) algorithms have been successfully applied to a range of challenging domains, from video games~\cite{mnih2013playing} to robotic control~\cite{schulman2015trust}. 

\vspace{-3pt}
Actor-critic algorithms are among the most popular approaches in DRL, e.g. \emph{DDPG}~\cite{lillicrap2015continuous}, \emph{TRPO}~\cite{schulman2015trust}, \emph{TD3}~\cite{fujimoto2018addressing} and \emph{SAC}~\cite{haarnoja2018soft}. These methods are based on policy iteration, which alternates between policy evaluation and policy improvement~\cite{sutton1998introduction}. Actor-critic methods jointly optimize the value function (critic) and the policy (actor) as it is often impractical to run either of these to convergence~\cite{haarnoja2018soft}. 

\vspace{-3pt}
In DRL, both the actor and critic use deep neural networks as the function approximator. However, DRL is known to assign unrealistically high values to state-action pairs represented by the Q-function. This is detrimental to the quality of the greedy control policy derived from Q~\cite{van2018deep}. Mnih~\etal~\cite{mnih2015human} proposed to use a  {\em target network\/} to mitigate divergence. A target network is a copy of the current Q function that is held fixed to serve as a stable target within the TD error update. The parameters of the target network are either infrequently copied~\cite{mnih2015human} or obtained by Polyak averaging~\cite{lillicrap2015continuous}. A limitation of using a target network is that it can slow down learning due to delayed function updates. We propose an approach that reduces the need for a target network in DRL while still ensuring stable learning and good performance in high-dimensional domains. We add a self-regularization term to encourage small changes to the target value while minimizing the Temporal Difference~(TD)-error~\cite{sutton1998introduction}. 

\vspace{-3pt}
Evolution Strategies~(ES) are a family of black-box optimization algorithms which are typically very stable, but scale poorly in high-dimensional search spaces, (e.g. neural networks)~\cite{pourchot2018cemrl}. Gradient-based DRL methods are often sample efficient, particularly in the off-policy setting when, unlike evolutionary search methods, they can continue to sample previous experiences to improve value estimation. But these approaches can also be unstable and highly sensitive to hyper-parameter tuning~\cite{pourchot2018cemrl}. We propose a novel policy improvement method which combines both approaches to get the best of both worlds. Specifically, after the actor network first outputs an initial action, we apply the {\em Cross Entropy Method\/}~(CEM)~\cite{Rubinstein1999CEM}~\cite{shao2020learning}~\cite{shao2020concept} to search the neighborhood of the initial action to find a second action associated with a higher Q value. Then we leverage the second action in the policy improvement stage to speed up the learning process. 

To mitigate the overestimation issue in Q learning~\cite{Thrun+Schwartz:1993}, Fujimoto \etal~\cite{fujimoto2018addressing} proposed Clipped Double Q-Learning in which the authors learn two Q-functions and use the smaller one to form the targets in the TD-Learning process. This method may suffer from under-estimation~\cite{fujimoto2018addressing}. In practice, we also observe that the discrepancy between the two Q-functions can increase dramatically which hinders the learning process. We propose Max-min Double Q-Learning to address this discrepancy. 

Our main contribution in this work is a \textbf{self-regularized} TD-learning method to address divergence without requiring a target network that may slow down learning progress. In addition, we propose a \textbf{self-guided} policy improvement method which combines policy-gradients and zero-order optimization. This helps to speed up learning and is robust to local noise in the Q function approximation. Taken together, these components define \emph{GRAC}, a novel self-guided and self-regularized actor critic algorithm. We evaluate \emph{GRAC} on six continuous control domains from OpenAI gym~\cite{brockman2016openai}, where we achieve or outperform state of the art result in every environment tested. We run our experiments across a large number of seeds with fair evaluation metrics~\cite{duan2016benchmarking}, perform extensive ablation studies, and open source both our code and learning curves.

\section{Related Work}
The proposed algorithm incorporates three key ingredients within the actor-critic method: a self-regularized TD update, self-guided policy improvements based on evolution strategies, and  Max-min double Q-Learning. In this section, we review prior work related to these ideas.
\vspace{-7pt}
\paragraph{Divergence in Deep Q-Learning}
In Deep Q-Learning, we use a nonlinear function approximator such as a neural network to approximate the Q-function that represents the value of each state-action pair. Learning the Q-function in this way is known to suffer from divergence issues~\cite{tsitsiklis1997analysis} such as assigning unrealistically high values to state-action pairs~\cite{van2018deep}. This is detrimental to the quality of the greedy control policy derived from Q~\cite{van2018deep}. To mitigate the divergence issue, Mnih~\etal\cite{mnih2015human} introduce a target network which is a copy of the estimated Q-function and is held fixed to serve as a stable target for some number of steps.  However, target networks can slow down the learning process due to delayed function updates~\cite{ijcai2019-379}. Durugkar~\etal~\cite{NIPS17-ishand} propose Constrained Q-Learning, which uses a constraint to prevent the average target value from changing after an update. Achiam~\etal\cite{achiam2019towards} give a simple analysis based on a linear approximation of the Q function and develop a stable Deep Q-Learning algorithm for continuous control without target networks. However, their proposed method requires separately calculating backward passes for each state-action pair in the batch, and solving a system of equations at each timestep. 
Peng~\etal~\cite{peng2019advantage} uses Monte Carlo instead of TD error to update the Q network and removes the need for a target network. Our proposed \emph{GRAC} algorithm adds a self-regularization term to the TD-Learning objective to keep the change of the state-action value small. 

\vspace{-4pt}
\paragraph{Evolution Strategies in Deep Reinforcement Learning}
{\em Evolution Strategies\/} (ES) are a family of black-box optimization algorithms which are typically very stable, but scale poorly in high-dimensional search spaces. Gradient-based deep RL methods, such as \emph{DDPG}~\cite{lillicrap2015continuous}, are often sample efficient, particularly in the off-policy setting. These off-policy methods can continue to reuse previous experience to improve value estimations but can be unstable and highly sensitive to hyper-parameter tuning~\cite{pourchot2018cemrl}. Researchers have proposed to combine these approaches to get the best of both worlds. Pourchot~\etal~\cite{pourchot2018cemrl} proposed \emph{CEM-RL} to combine CEM with either \emph{DDPG}~\cite{lillicrap2015continuous} or \emph{TD3}~\cite{fujimoto2018addressing}. However, \emph{CEM-RL} applies \emph{CEM} within the actor parameter space which is extremely high-dimensional, making the search not efficient. Kalashnikov~\etal~\cite{kalashnikov2018scalable} introduce \emph{QT-Opt}, which leverages \emph{CEM} to search the landscape of the Q function, and enables Q-Learning in continuous action spaces without using an actor. Based on \emph{QT-Opt}, Simmons-Edler~\etal~\cite{simmons2019q} leverage CEM to search the landscape the Q-function but propose to initialize CEM with a Gaussian prior covering the action space, independent of observations. However, as shown in~\cite{yan2019learning}, \emph{CEM} does not scale well to high-dimensional action spaces, such as in the Humanoid task used in this paper. We first let the actor network output an initial Gaussian action distribution conditioned on current state. Then we use CEM to search for an action with a higher Q value than the Q value of the Gaussian mean. Starting from the predicted distribution, \emph{GRAC} speeds up the learning process compared to popular actor-critic methods. 

\vspace{-4pt}
\paragraph{Double-Q Learning}
Using function approximation, Q-learning~\cite{Watkins:89} is known to suffer from overestimation~\cite{Thrun+Schwartz:1993}. To mitigate this problem, Hasselt~\etal~\cite{hasselt2010double} proposed \emph{Double Q-learning} which uses two Q functions with independent sets of weights. 
\emph{TD3}~\cite{fujimoto2018addressing} proposed {\em Clipped Double Q-learning\/} to learn two Q-functions and uses the smaller of the two to form the targets in the TD-Learning process. However, \emph{TD3}~\cite{fujimoto2018addressing} may lead to underestimation. Besides, the actor network in ~\emph{TD3}~\cite{fujimoto2018addressing} is trained to select the action to maximize the first Q function throughout the training process which may make it very different from the second Q-function. A large discrepancy results in large TD-errors which in turn results in large gradients during the update of the actor and critic networks. This makes instability of the learning process more likely. We propose {\em Max-min Double Q-Learning\/} to balance the differences between the two Q functions and provide a better approximation of the Bellman optimality operator~\cite{sutton1998introduction}.

\section{Preliminaries}
\vspace{-4pt}
In this section, we define the notation used in subsequent sections. Consider a {\em Markov Decision Process\/}~(MDP), defined by the tuple $(\mathcal{S},\mathcal{A},\mathscr{P},r,\rho_0,\gamma)$, where $\mathcal{S}$ is a finite set of states, $\mathcal{A}$ is a finite set of actions, $\mathscr{P}:\mathcal{S} \times \mathcal{A} \times \mathcal{S} \rightarrow \mathbb{R}$ is the transition probability distribution, $r:\mathcal{S} \times \mathcal{A} \rightarrow \mathbb{R}$ is the reward function, $\rho_0: \mathcal{S} \rightarrow \mathbb{R}$ is the distribution of the initial state $s_0$, and $\gamma \in (0,1)$ is the discount factor. At each discrete time step $t$, with a given state $s_{t} \in \mathcal{S}$, the agent selects an action $a_t \in \mathcal{A}$, receiving a reward $r$ and the new state $s_{t+1}$ of the environment.

Let $\pi$ denote the policy which maps a state to a probability distribution over the actions, $\pi: \mathcal{S} \rightarrow \mathcal{P(\mathcal{A}})$. The return from a state is defined as the sum of discounted reward $R_t = \sum_{i=t}\gamma^{i-t}r(s_i,a_i)$. In reinforcement learning, the objective is to find the optimal policy $\pi^*$, with parameters $\phi$, which maximizes the expected return $J(\phi)=\sum_{t}\E_{(s_t,a_t) \sim \rho_\pi(s_t,a_t)}[\gamma^t r(s_t,a_t)]$ where $\rho_{\pi}(s_t)$ and $\rho_{\pi}(s_t,a_t)$ denote the state and state-action marginals of the trajectory distribution induced by the policy $\pi(a_t|s_t)$.

We use the following standard definitions of the state-action value function $Q_\pi$. It describes the expected discounted reward after taking an action $a_t$ in state $s_t$ and thereafter following policy $\pi$:
\begin{equation}
    Q_{\pi}(s_t,a_t) = \E_{ \pi}[R_t|s_t, a_t].
\end{equation}


In this work we use \emph{CEM} to find optimal actions with maximum Q values. \emph{CEM} is a randomized zero-order optimization algorithm. To find the action $a$ that maximizes $Q(s,a)$, \emph{CEM} is initialized with a paramaterized distribution over $a$, $P(a; \psi)$. Then it iterates between the following two steps~\cite{Botev2013CEM}:
First generate $a_1, \dots, a_N \sim P(s; \psi)$. Retrieve their Q values $Q(s, a_i)$ and sort the actions to have decreasing Q values. Then keep the first $K$ actions, and solve for an updated parameters $\psi^{\prime}$:
\[
\psi^{\prime} = \text{argmax}_{\psi} \frac{1}{K}\sum_
{i=1}^K \log(P(a_i;\psi))
\]
In the following sections, we denote \emph{CEM}$(Q(s,\cdot), \pi(\cdot | s))$ as the action found by \emph{CEM} to maximize $Q(s,\cdot)$, when \emph{CEM} is intiailized by the distribution predicted by the policy. 

\section{Technical Approach}
\begin{algorithm}
{Initialize critic network $Q_{\theta 1}$, $Q_{\theta 2}$ and actor network $\pi_\phi$ with random parameters $\theta 1$, $\theta 2$ and $\phi$}\\
{Initialize replay buffer $\mathcal{B}$, Set $\alpha < 1$}
\begin{algorithmic}[1]
	\For {$i=1, \dots$}
		\State Select action $a \sim \pi_{\phi_i}(s)$ and observe reward $r$ and new state $s^{\prime}$
		\State Store transition tuple  $(s,a,r,s^{\prime})$ in $\mathcal{B}$
		\State Sample mini-batch of $N$ transitions $(s_t,a_t,r_t,s_{t+1})$ from $\mathcal{B}$
		\State $\hat{a}_{t+1} \sim \pi_{\phi_i}(s_{t+1})$
		\State $\tilde{a}_{t+1} \leftarrow$
		\emph{CEM}($Q(s_{t+1},\cdot;\theta_2),\pi_{\phi_i}(\cdot | s_{t+1})$)
		\State $y \leftarrow r_t + \gamma  \max\{\min_{j=1,2} Q(s_{t+1},\tilde{a}_{t+1};\theta_j), \min_{j=1,2} Q(s_{t+1},\hat{a}_{t+1};\theta_j)\}$
		\State $a^{\dagger} \leftarrow \argmax_{\{\tilde{a},\hat{a}\}}\{\min_{j=1,2} Q(s_{t+1},\tilde{a}_{t+1};\theta_j), \min_{j=1,2} Q(s_{t+1},\hat{a}_{t+1};\theta_j)\}$
		\State $y_1^\prime,y_2^\prime \leftarrow Q(s_{t+1},a^{\dagger};\theta_1), Q(s_{t+1},a^{\dagger};\theta_2)$
		\For {$k=1$ to $K$}
            \State $\mathcal{L}_{k} = \|y - Q(s_t,a_t;\theta_1)\|_2^2 + \|y - Q(s_t,a_t;\theta_2)\|_2^2 + \|y_1^\prime - Q(s_{t+1},a^{\dagger};\theta_1) \|_2^2 + \|y_2^\prime - Q(s_{t+1},a^{\dagger};\theta_2) \|_2^2$ 
            \State $\theta1 \leftarrow \theta1 - \lambda \nabla_{\theta 1}\mathcal{L}_{k}$, $\theta 2 \leftarrow \theta 2 - \lambda \nabla_{\theta 2}\mathcal{L}_{k}$
            \If{$\mathcal{L}_{k} < \alpha \mathcal{L}_{1}$}
                \State Break
            \EndIf
        \EndFor
    \State $\hat{a}_t \sim \pi_{\phi_i}(s_t)$
    \State 
    $J_{\pi}(\phi) = \E_{(s_t,\hat{a}_t)}[Q(s_t, \hat{a}_t;\theta_1)]$ 
     \State 
     $\bar{a}_t \leftarrow$
		CEM($Q(s_t,\cdot;\theta_1), \pi_{\phi_i}(\cdot|s_t)$)
    \State $
\phi \leftarrow \phi - \lambda \nabla_{\phi} J_{\pi}(\phi ) - \lambda \E_{(s_t,\hat{a}_t)} [Q(s_t,\bar{a}_t;\theta_1)-Q(s_t,\hat{a}_t;\theta_1)]_{+} \nabla_{\phi} \log\pi(\bar{a}_t|s_t; \phi)$
	\EndFor
\end{algorithmic} 
\caption{GRAC}\label{alg:alg1}
\end{algorithm}

\vspace{-6pt}
\subsection{Self-Regularized TD Learning}\label{sec:target}
Reinforcement learning is prone to instability and divergence when a nonlinear function approximator such as a neural network is used to represent the Q function~\cite{tsitsiklis1997analysis}. Mnih~\etal\cite{mnih2015human} identified several reasons for this. 
One is the correlation between the current action-values and the target value. Updates to $Q(s_t,a_t)$ often also increase $Q(s_{t+1},a^{*}_{t+1})$ where $a^{*}_{t+1}$ is the optimal next action. Hence, these updates also increase the target value $y_t$ which may lead to oscillations or the divergence of the policy. 


More formally, given transitions $(s_t,a_t,r_t,s_{t+1})$ sampled from the replay buffer distribution $\mathcal{B}$, the Q network can be trained by minimising the loss functions $\mathcal{L}(\theta_i)$ at iteration $i$:
\begin{equation}
    \mathcal{L}(\theta_i) = \E_{(s_t,a_t) \sim \mathcal{B}}\|(Q(s_t,a_t; \theta_i) - y_i)\|^2
\end{equation}
where for now let us assume $y_i= r_t + \gamma\max_{a}Q(s_{t+1},a;\theta_i)$ to be the target for iteration $i$ computed based on the current Q network parameters $\theta_i$.  ${a}^{*}_{t+1}=\argmax_{a}Q(s_{t+1},a)$. If we update the parameter $\theta_{i+1}$ to reduce the loss $\mathcal{L}(\theta_i)$, it changes both $Q(s_t,a_t;\theta_{i+1})$ and $Q(s_{t+1},a_{t+1}^{*};\theta_{i+1})$. Assuming an increase in both values, then the new target value $y_{i+1} = r_t + \gamma Q(s_{t+1},a^*_{t+1}; \theta_{i+1})$ for the next iteration will also increase leading to an explosion of the Q function. We demonstrated this behavior in an ablation experiment with results in Fig.~\ref{fig:abl1}. We also show how maintaining a separate target network~\cite{mnih2015human} with frozen parameters $\theta^-$ to compute $y_{i+1} = r_t + \gamma Q(s_{t+1},a^*_{t+1}; \theta^-)$ delays the update of the target and therefore leads to more stable learning of the Q function. However, delaying the function updates also comes with the price of slowing down the learning process. 

We propose a self-Regularized TD-learning approach to minimize the TD-error while also keeping the change of $Q(s_{t+1},a^*_{t+1})$ small. This regularization mitigates the divergence issue~\cite{tsitsiklis1997analysis}, and no longer requires a target network that would otherwise slow down the learning process. Let $y^{\prime}_i = Q(s_{t+1},a_{t+1}^{*}; \theta_i)$, and $y_i=r_t + \gamma y^{\prime}_i$. We define the learning objective as
\begin{equation}
    \min_{\theta} \|Q(s_t,a_t;\theta) - y_i\|^2 + \|Q(s_{t+1},a^{*}_{t+1};\theta) - y^\prime_i\|^2
\end{equation}
where the first term is the original TD-Learning objective and the second term is the regularization term penalizing large updates to $Q(s_{t+1}, a^*_{t+1})$. Note that when the current Q network updates its parameters $\theta$, both $Q(s_t,a_t)$ and $Q(s_
{t+1},a^{*}_{t+1})$ change. Hence, the target value $y_i$ will also change which is different from the approach of keeping a frozen target network for a few iterations.  We will demonstrate in our experiments that this self-regularized TD-Learning approach removes the delays in the update of the target value thereby achieves faster and stable learning. 

\subsection{Self-Guided Policy Improvement with Evolution Strategies}\label{sec:cem_a}

The policy, known as the actor, can be updated through a combination of two parts. The first part, which we call Q-loss policy update, improves the policy through local gradients of the current Q function, while the second part, which we call \emph{CEM} policy update, finds a high-value action via \emph{CEM} in a broader neighborhood of the Q function landscape, and update the action distribution to concentrate towards this high-value action. We describe the two parts formally below.

Given states $s_t$ sampled from the replay buffer, the Q-loss policy update maximizes the objective
\begin{equation}
    J_{\pi}(\phi) = \E_{s_t \sim \mathcal{B},\hat{a}_t \sim \pi}[Q(s_t, \hat{a}_t)],
\end{equation}
where $\hat{a}_t$ is sampled from the current policy $\pi(\cdot | s_t)$.
The gradient is taken through the reparameterization trick. We reparameterize the policy using a neural network transformation as described in Haarnoja~\etal~\cite{haarnoja2018soft}, 
\begin{equation}
    \hat{a}_t = f_{\phi}(\epsilon_t | s_t)
\end{equation}
where $\epsilon_t$ is an input noise vector, sampled from a fixed distribution, such as a standard multivariate Normal distribution.
Then the gradient of $J_\pi(\phi)$ is:
\begin{equation}
    \nabla J_{\pi}(\phi) = \E_{s_t \sim \mathcal{B}, \epsilon_t \sim \mathcal{N}}[\frac{\partial Q(s_t,f_{\phi}(\epsilon_t | s_t))} {\partial f} \frac{\partial f_\phi(\epsilon_t | s_t)}{\partial \phi}]
\end{equation}

For the CEM policy update, given a minibatch of states $s_t$, we first find a high-value action $\bar{a}_t$ for each state by running \emph{CEM} on the current Q function, $\bar{a}_t = \text{\emph{CEM}}(Q(s_t,\cdot), \pi(\cdot | s_t))$. Then the policy is updated to increase the probability of this high-value action.
The guided update on the parameter $\phi$ of $\pi$ at iteration $i$ is
\begin{equation}
  \E_{s_t \sim \mathcal{B}, \hat{a}_t \sim \pi} [Q(s_t,\bar{a}_t) - Q(s_t,\hat{a}_t)]_{+} \nabla_{\phi} \log\pi_i(\bar{a}_t|s_t).
\end{equation}

We used $Q(s_t, \hat{a}_t)$ as a baseline term, since its expectation over actions $\hat{a}_t$ will give us the normal baseline $V(s_t)$:
\begin{equation}
    \E_{s_t\sim \mathcal{B}} [Q(s_t,\bar{a}_t)-V(s_t)] \nabla_{\phi} \log\pi_i(\bar{a}_t|s_t)
\end{equation}

In our implementation, we only perform an update if the improvement on the $Q$ function, $Q(s_t,\bar{a}_t)-Q(s_t,\hat{a}_t)$, is non-negative, to guard against the occasional cases where \emph{CEM} fails to find a better action.

Combining both parts of updates, the final update rule on the parameter $\phi_i$ of policy $\pi_i$ is 
\[
\phi_{i+1} = \phi_{i} - \lambda \nabla_{\phi} J_{\pi_i}(\phi_{i} ) - \lambda \E_{s_t \sim \mathcal{B}, \hat{a}_t \sim \pi_i} [Q(s_t,\bar{a}_t) - Q(s_t,\hat{a}_t)]_{+} \nabla_{\phi} \log\pi_i(\bar{a}_t|s_t)
\]
where $\lambda$ is the step size.

We can prove that if the Q function has converged to $Q^{\pi}$, the state-action value function induced by the current policy, then both the Q-loss policy update and the \emph{CEM} policy update will be guaranteed to improve the current policy.
We formalize this result in Theorem~\ref{theo:Qloss} and Theorem~\ref{theo:CEMImpro}, and prove them in Appendix~\ref{appendsubsec:theorem1} and ~\ref{appendsubsec:theorem2}.
\begin{theorem}\label{theo:Qloss}
\textbf{$Q$-loss Policy Improvement} Starting from the current policy $\pi$, we maximize the objective $J_\pi = \E_{(s,a) \sim \rho_\pi(s,a)} Q^{\pi}(s,a)$. The maximization converges to a critical point denoted as $\pi_{new}$. Then the induced Q function, $Q^{\pi_{new}}$, satisfies $\forall (s,a), Q^{\pi_{new}}(s,a) \geq Q^{\pi}(s,a).$
\end{theorem}

\begin{theorem}\label{theo:CEMImpro}
\textbf{\emph{CEM} Policy Improvement} Assuming the \emph{CEM} process is able to find the optimal action of the state-action value function, $a^*(s) = \argmax_{a}Q^{\pi}(s,a)$, where $Q^{\pi}$ is the Q function induced by the current policy $\pi$. 
By iteratively applying the update $ \E_{(s,a) \sim \rho_\pi(s,a)} [Q(s,a^*)-Q(s,a)]_{+}\nabla \log\pi(a^*|s)$, the policy converges to $\pi_{new}$. Then $Q^{\pi_{new}}$ satisfies $\forall (s,a), Q^{\pi_{new}}(s,a) \geq Q^{\pi}(s,a).$
\end{theorem}

\subsection{Max-min Double Q-Learning}\label{sec:cem_c}
Q-learning~\cite{Watkins:89} is known to suffer from overestimation~\cite{Thrun+Schwartz:1993}. Hasselt~\etal~\cite{hasselt2010double} proposed Double-Q learning which uses two Q functions with independent sets of weights to mitigate the overestimation problem. Fujimoto~\etal~\cite{fujimoto2018addressing} proposed Clipped Double Q-learning with two Q function denoted as $Q(s,a;\theta_1)$ and $Q(s,a;\theta_2)$, or $Q_1$ and $Q_2$ in short. Given a transition $(s_t,a_t,r_t,s_{t+1})$, Clipped Double Q-learning uses the minimum between the two estimates of the Q functions when calculating the target value in TD-error~\cite{sutton1998introduction}:
\begin{equation}
    y= r_t + \gamma\min_{j=1,2} Q(s_{t+1},\hat{a}_{t+1};\theta_j)
    \label{eqn:clipped}
\end{equation}
where $\hat{a}_{t+1}$ is the predicted next action.

Fujimoto~\etal~\cite{fujimoto2018addressing} mentioned that such an update rule may induce an underestimation bias. In addition, $\hat{a}_{t+1}=\pi_{\phi}(s_{t+1})$ is the prediction of the actor network. The actor network's parameter ${\phi}$ is optimized according to the gradients of $Q_1$. In other words, $\hat{a}_{t+1}$ tends be selected according to the $Q_1$ network which consistently increases the discrepancy between the two Q-functions. In practice, we observe that the discrepancy between the two estimates of the Q-function, $|Q_1 - Q_2|$, can increase dramatically leading to an unstable learning process. An example is shown in Fig.~\ref{fig:abl2} where $Q(s_{t+1},\hat{a}_{t+1};\theta_1)$ is always bigger than $Q(s_{t+1},\hat{a}_{t+1};\theta_2)$. 

We introduce \emph{Max-min Double Q-Learning} to reduce the discrepancy between the Q-functions. We first select $\hat{a}_{t+1}$ according to the actor network $\pi_{\phi}(s_{t+1})$. Then we run \emph{CEM} to search the landscape of $Q_2$ within a broad neighborhood of $\hat{a}_{t+1}$ to return a second action $\tilde{a}_{t+1}$. Note that \emph{CEM} selects an action $\tilde{a}_{t+1}$ that maximises $Q_2$ while the actor network selects an action $\hat{a}_{t+1}$ that maximises $Q_1$. We gather four different Q-values: $Q(s_{t+1},\hat{a}_{t+1};\theta_1)$, $Q(s_{t+1},\hat{a}_{t+1};\theta_2)$, $Q(s_{t+1},\tilde{a}_{t+1};\theta_1)$, and $Q(s_{t+1},\tilde{a}_{t+1};\theta_2)$. We then run a max-min operation to compute the target value  that cancels the biases induced by $\hat{a}_{t+1}$ and $\tilde{a}_{t+1}$.
\begin{equation}
\label{eq:maxmin}
\begin{aligned}
y & = r_t + \gamma  \max\{\min_{j=1,2} Q(s_{t+1},\hat{a}_{t+1};\theta_j), \min_{j=1,2} Q(s_{t+1},\tilde{a}_{t+1};\theta_j)\}
\end{aligned}
\end{equation} 

The inner min-operation $\min_{j=1,2} Q(s_{t+1},\hat{a}_{t+1};\theta_j)$ is adopted from Eq.~\ref{eqn:clipped} and mitigates overestimation~\cite{Thrun+Schwartz:1993}. The outer max operation helps to reduce the difference between $Q_1$ and $Q_2$. In addition, the max operation provides a better approximation of the Bellman optimality operator~\cite{sutton1998introduction}. We visualize $Q_1$ and $Q_2$ during the learning process in Fig.~\ref{fig:abl2}.
We have a theorem formalizing the convergence of the proposed Max-min Double Q-Learning approach in the finite MDP setting. We state the theorem and its prove in Appendix~\ref{appendsubsec:theorem3}.


\section{Experiments}
\subsection{Comparative Evaluation}
We present \emph{GRAC}, a self-guided and self-regularized actor-critic algorithm as summarized in Algorithm~\ref{alg:alg1}. To evaluate {\em GRAC}, we measure its performance on the suite of MuJoCo continuous control tasks~\cite{todorov2012mujoco}, interfaced through OpenAI Gym~\cite{brockman2016openai}. We compare our method with \emph{DDPG}~\cite{lillicrap2015continuous}, \emph{TD3}~\cite{fujimoto2018addressing}, \emph{TRPO}~\cite{schulman2015trust}, and \emph{SAC}~\cite{haarnoja2018soft}. 
We use the source code released by the original authors and adopt the same hyperparameters reported in the original papers and the number of training steps according to \emph{SAC}~\cite{haarnoja2018soft}. Hyperparameters for all experiments are in Appendix~\ref{appendsubsec:hyperparams}. 
Results are shown in Figure \ref{fig:compare_results}. {\em GRAC} outperforms or is comparable to all other algorithms in both final performance and learning speed across all tasks. In addition, we apply the self-regularized TD-learning method to DQN~\cite{mnih2015human} on the Atari Breakout environment and it outperforms \emph{DQN} by 25\%. The learning curve over 50 million steps is shown in Appendix~\ref{appendsec:atari}.

\begin{figure*}[ht!]
\centering
\includegraphics[width=.85\linewidth]{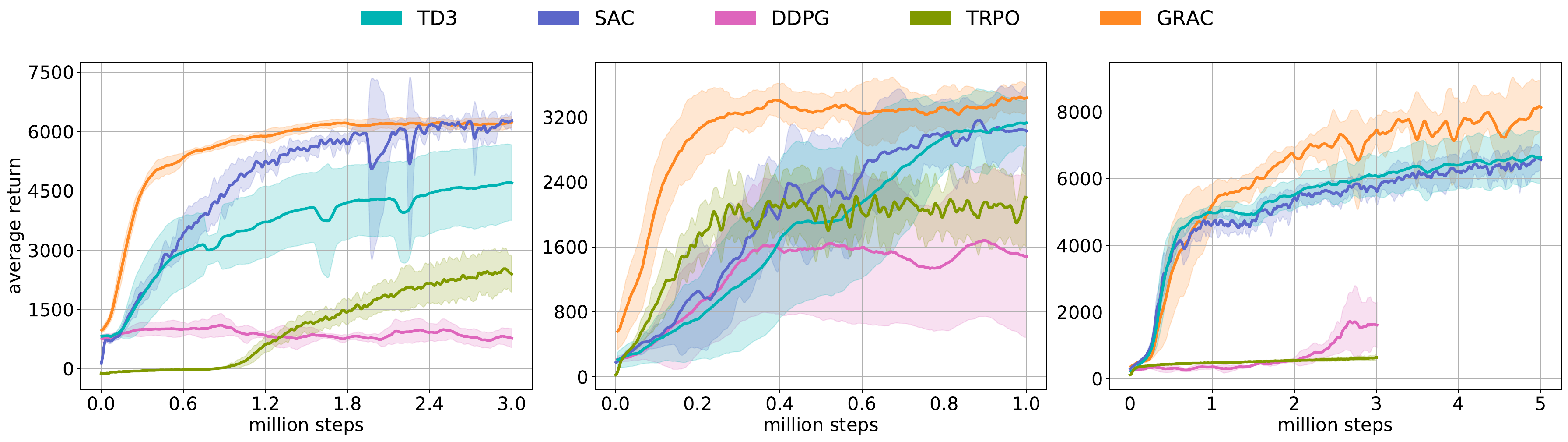}
\begin{tabular}{ccc}
    \begin{minipage}{.2\textwidth}
    \centering
    \par\small{(a)~Ant-v2}
    \end{minipage}
&
    \begin{minipage}{.3\linewidth}
    \centering
    \par\small{(b)~Hopper-v2}
   \end{minipage}
& 
    \begin{minipage}{.2\linewidth}
    \centering
    \par\small{(c)~Humanoid-v2}
    \end{minipage}
\end{tabular}
\includegraphics[width=.85\linewidth]{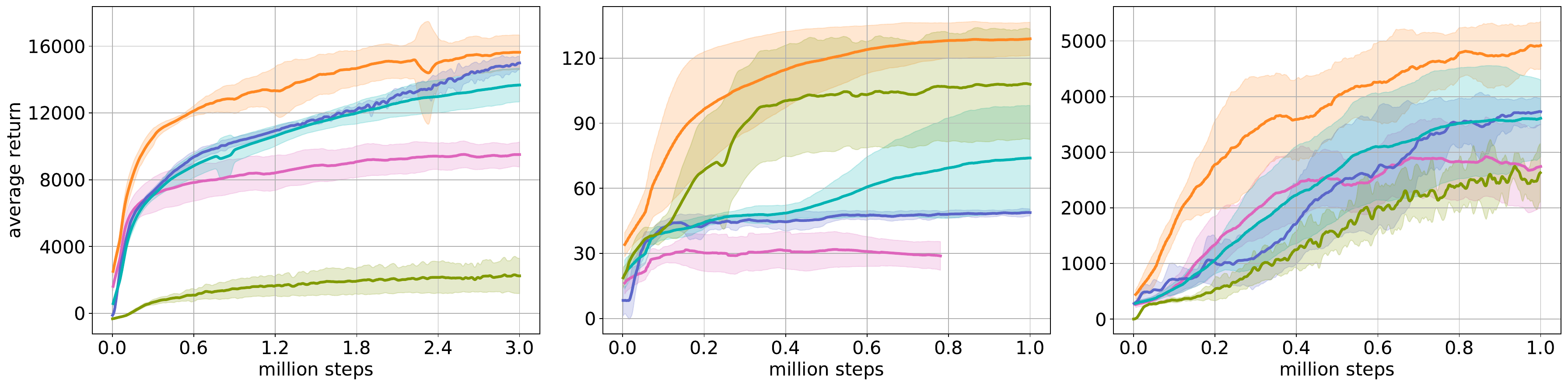}
\begin{tabular}{ccc}
    \begin{minipage}{.2\linewidth}
    \centering
    \par\small{(d)~HalfCheetah-v2}
    \end{minipage}
&
 	\begin{minipage}{.3\linewidth}
    \centering
    \par\small{(e)~Swimmer-v2}
   \end{minipage}
& 
    \begin{minipage}{.2\linewidth}
    \centering
    \par\small{(f)~Walker2d-v2}
    \end{minipage}\\ 
\end{tabular}
\caption{Learning curves for the OpenAI gym continuous control tasks. For each task, we train 8 instances of each algorithm, using 8 different seeds. Evaluations are performed every 5000 interactions with the environment. Each evaluation reports the return (total reward), averaged over 10 episodes. For each training seed, we use a different seed for evaluation, which results in different start states. The solid curves and shaded regions represent the mean and standard deviation, respectively, of the average return over 8 seeds. All curves are smoothed with window size 10 for visual clarity. \emph{GRAC} (orange) learns faster than other methods across all tasks. \emph{GRAC} achieves comparable result to the state-of-the-art methods on the Ant-v2 task and outperforms prior methods on the other five tasks including the complex high-dimensional Humanoid-v2.}\label{fig:compare_results}
\end{figure*}


\subsection{Ablation Study}
 In this section, we present ablation studies to understand the contribution of each proposed component: Self-Regularized TD-Learning~(Section~\ref{sec:target}), Self-Guided Policy Improvment~(Section~\ref{sec:cem_a}), and Max-min Double Q-Learning~(Section~\ref{sec:cem_c}). We present our results in Fig.~\ref{fig:table2} in which we compare the performance of {\em GRAC} with alternatives, each removing one component from GRAC. Additional learning curves can be found in Appendix~\ref{appendsubsec:ablation_evolutionary} and \ref{appendsubsec:ablation_self_reg}. We also run experiments to examine how sensitive GRAC is to some hyperparameters such as $\alpha$ and $K$ listed in Alg.~\ref{alg:alg1}, and the results can be found in Appendix~\ref{appendsubsec:termination}.

\paragraph{Self-Regularized TD Learning} To verify the effectiveness of the proposed self-regularized TD-learning method, we apply our method to \emph{DDPG} (DDPG w/o target network w/ target regularization). We compare against two baselines: the original \emph{DDPG} and \emph{DDPG} without target networks for both actor and critic (\emph{DDPG w/o target network}). We choose DDPG, because it does not have additional components such as Double Q-Learning, which may complicate the analysis of this comparison.

In Fig.~\ref{fig:abl1}, we visualize the average returns, and average $Q_1$ values over training batchs (${y^\prime}_1$ in Alg.\ref{alg:alg1}). The $Q_1$ values of \emph{DDPG w/o target network} changes dramatically which leads to poor average returns. \emph{DDPG} maintains stable Q values but makes slow progress. Our proposed \emph{DDPG} w/o target network w/ target regularization maintains stable Q values. In addition, we compare the average returns of \emph{DDPG w/o target network, w/ target regularization} and \emph{DDPG} within one million steps over four random seeds. \emph{DDPG w/o target network, w/ target regularization} outperforms \emph{DDPG} by large margins in five out of six Mujoco tasks (Fig.~\ref{fig:abl1} shows results on Hopper-v2. The remaining results can be found in Appendix~\ref{appendsubsec:ablation_self_reg}). This demonstrates the effectiveness of our method and its potentials to be applied to a wide range of DRL methods. 

\begin{figure}[bh!]
\centering
\begin{tabular}{cc}
    \begin{minipage}{.45\textwidth}
    \centering
    \includegraphics[width=\linewidth]{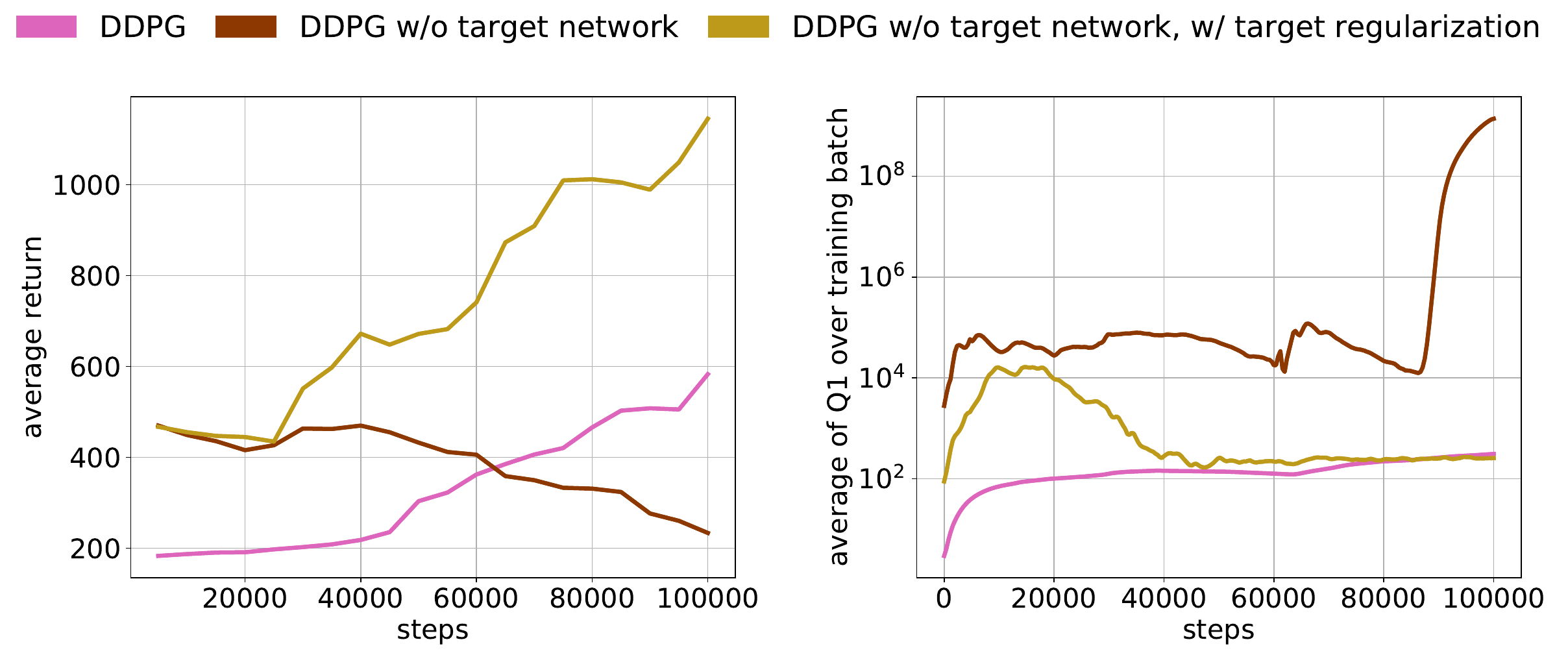}
    \end{minipage}
    &
 \begin{minipage}{.36\textwidth}
    \centering
    \includegraphics[width=\linewidth]{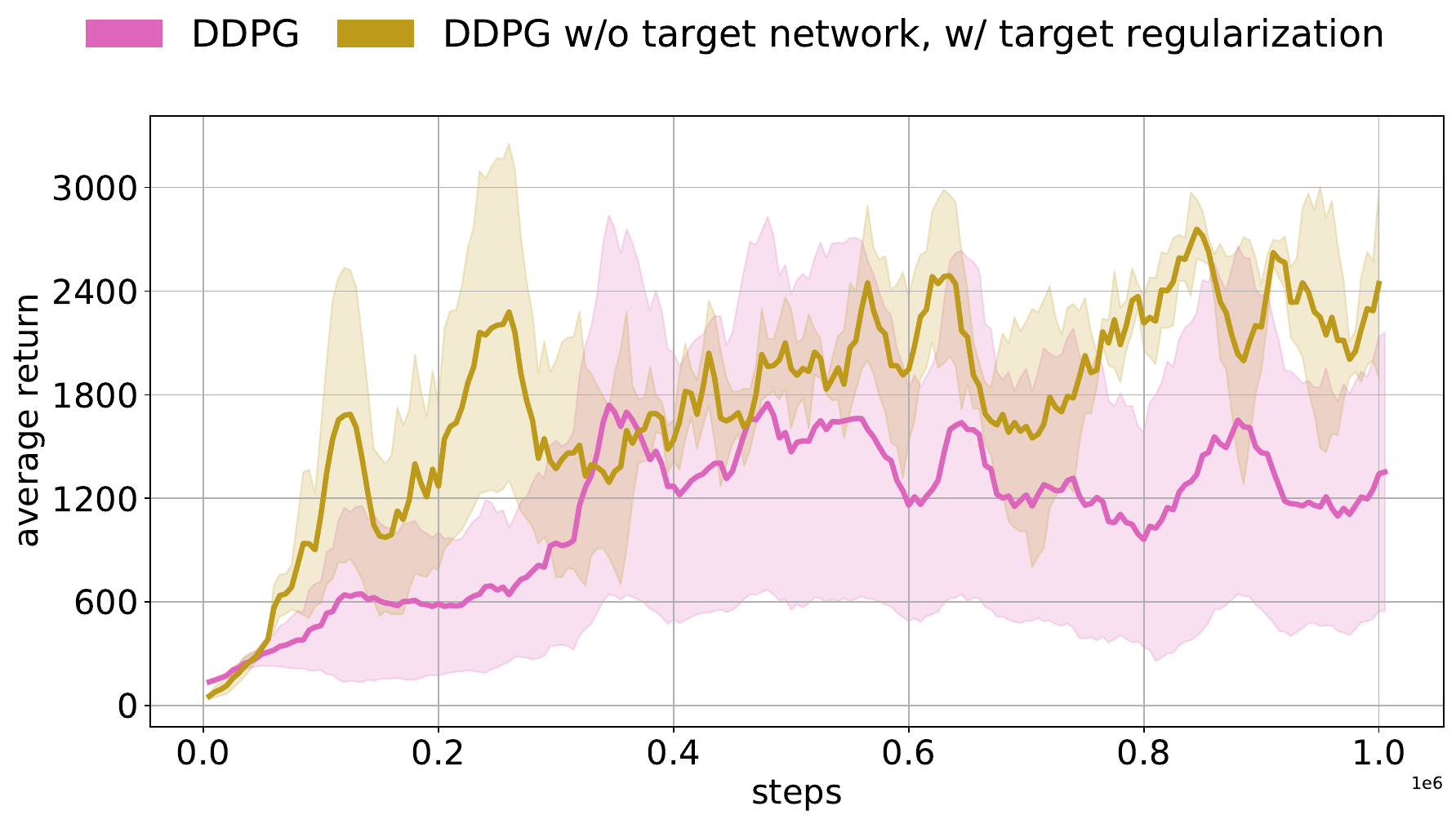}
    \end{minipage}
    
\end{tabular}

\begin{tabular}{ccc}
    \begin{minipage}{.2\textwidth}
    \centering
    \par\small{(a)~Returns on Hopper-v2}
    \end{minipage}
&
    \begin{minipage}{.257\linewidth}
    \centering
    \par\small{(b)~Average of $Q_1$ over training batch on Hopper-v2}
   \end{minipage}
  &
    \begin{minipage}{.3\linewidth}
    \centering
    \par\small{(c)~Returns on Hopper-v2 over one million steps}
   \end{minipage}
 \end{tabular}
 \caption{Learning curves and average $Q_1$ values ($y^{\prime}_1$ in Alg.~\ref{alg:alg1}) on Hopper-v2. \emph{DDPG} w/o target network quickly diverges as seen by the unrealistically high Q values. \emph{DDPG} is stable but progresses slowly. If we remove the target network and add the proposed target regularization, we both maintain stability and achieve faster learning than \emph{DDPG}.}\label{fig:abl1}
 \vspace{-4pt}
\end{figure}

\vspace{-4pt}   
\paragraph{Policy Improvement with Evolution Strategies}
The GRAC actor network uses a combination of two actor loss functions, denoted as \emph{QLoss} and \emph{CEMLoss}. \emph{QLoss} refers to the unbiased gradient estimators which extends the \emph{DDPG}-style policy gradients~\cite{lillicrap2015continuous} to stochastic policies. \emph{CEMLoss} represents the policy improvement guided by the action found with the zero-order optmization method CEM. We run another two ablation experiments on all six control tasks and compare it with our original policy training method denoted as \emph{GRAC}. As seen in Fig.\ref{fig:table2}, in general \emph{GRAC} achieves a better performance compared to either using \emph{CEMLoss} or \emph{QLoss}. The significance of the improvements varies within the six control tasks. For example, \emph{CEMLoss} plays a dominant role in Swimmer while \emph{QLoss} has a major effect in HalfCheetah. It suggests that \emph{CEMLoss} and \emph{QLoss} are complementary.

\begin{figure}[tb!]
    \centering
    \includegraphics[width=0.75\textwidth]{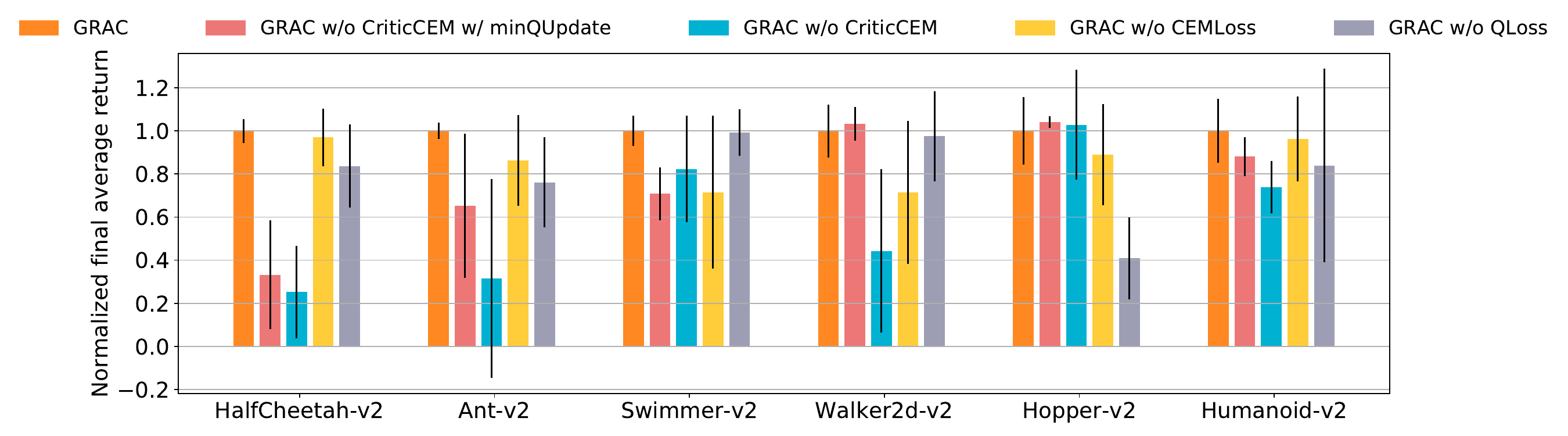}
    \caption{Final average returns, normalized w.r.t {\em GRAC} for all tasks. For each task, we train each ablation setting with 4 seeds, and average the last 10 evaluations of each seed (40 evaluations in total). The black lines represent one standard deviation. Actor updates without CEMLoss (\emph{GRAC w/o CEMLoss}) and actor updates w.r.t minimum of both Q networks (\emph{GRAC w/o CriticCEM w/ minQUpdate}) achieves slightly better performance on Walker2d-v2 and Hopper-v2. GRAC achieves the best performance on 4 out of 6 tasks, especially on more challenging tasks with higher-dimensional state and action spaces (Humanoid-v2, Ant-v2, HalfCheetah-v2). This suggests that individual components of GRAC complement each other.  }
    \label{fig:table2}
\end{figure}
\vspace{-3mm}

\paragraph{Max-min Double Q-Learning}
We additionally verify the effectiveness of the proposed Max-min Double Q-Learning method. We run an ablation experiment by replacing Max-min by Clipped Double Q-learing~\cite{fujimoto2018addressing} denoted as \emph{GRAC w/o CriticCEM}. In Fig.~\ref{fig:abl2}, we visualize the learning curves of the average return, $Q_1$ ($y^{\prime}_1$ in Alg.~\ref{alg:alg1}), and $Q_1-Q_2$ ($y^{\prime}_1-y^{\prime}_2$ in Alg.~\ref{alg:alg1}). \emph{GRAC} achieves high performance while maintaining a smoothly increasing Q function. Note that the difference between Q functions, $Q_1-Q_2$, remains around zero for \emph{GRAC}. \emph{GRAC w/o CriticCEM} shows high variance and drastic changes in the learned $Q_1$ value. In addition, $Q_1$ and $Q_2$ do not always agree. Such unstable Q values result in a performance crash during the learning process. Instead of Max-min Double Q Learning, another way to address the gap between $Q_1$ and $Q_2$ is to perform actor updates on the minimum of $Q_1$ and $Q_2$ networks (as seen in SAC). Replacing Max-min Double Q Learning with this trick achieves lower performance than \emph{GRAC} in more challenging tasks such as HalfCheetah-v2, Ant-v2, and Humanoid-v2 (See \emph{GRAC w/o CriticCEM w/ minQUpdate} in Fig.\ref{fig:table2}).

\begin{figure}[bh!]
\centering
\includegraphics[width=0.75\linewidth]{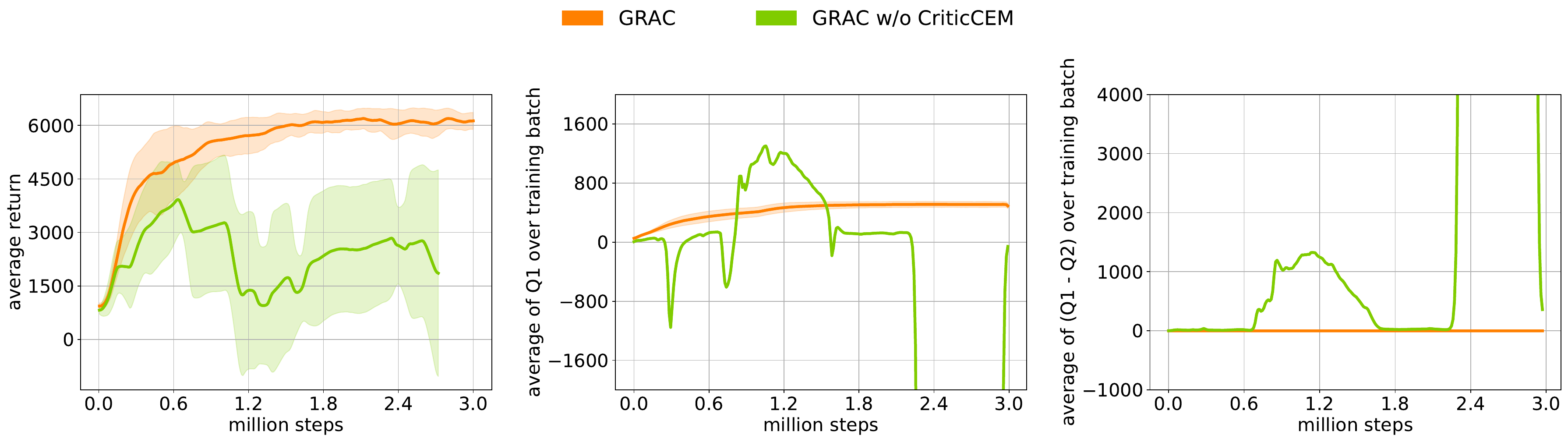}
\begin{tabular}{ccc}
    \begin{minipage}{.3\textwidth}
    \begin{flushright}
   \par\small{(a)~Returns on Ant-v2}
   \end{flushright}
    \end{minipage}
&
    \begin{minipage}{.3\linewidth}
    \centering
    \par \small{(b)~Average of $Q_1$ over training batch on Ant-v2}
   \end{minipage} 
&
    \begin{minipage}{.3\linewidth}
    \begin{flushleft}
    \par\small{(c)~Average of $Q_1-Q_2$ over training batch on Ant-v2}
    \end{flushleft}
   \end{minipage}
 \end{tabular}
 \caption{ Learning curves (left), average $Q_1$ values (middle), and average of the difference between $Q_1$ and $Q_2$ (right) on Ant-v2. Average Q values are computes as minibatch average of $y^{\prime}_1$ and $y^{\prime}_2$, defined in Alg.~\ref{alg:alg1}. \emph{GRAC w/o CriticCEM} represents replacing Max-min Double Q-Learning with Clipped Double Q-Learning. Without Max-min double Q-Learning to balance the magnitude of $Q_1$ and $Q_2$, $Q_1$ blows up significantly compared to $Q_2$, leading to divergence. }\label{fig:abl2}
\end{figure}
\vspace{-3mm}

\section{Conclusion}
\vspace{-6pt}
Leveraging neural networks as function approximators, DRL has been successfully demonstrated on a range of decision-making and control tasks. However, the nonlinear function approximators also introduce issues such as divergence and overestimation. As our main contribution, we proposed a self-regularized TD-learning method to address divergence without requiring a target network that may slow down learning progress. The proposed method is agnostic to the specific Q-learning method and can be added to any of them. In addition, we propose self-guided policy improvement by combining policy-gradient with zero-order optimization such as the Cross Entropy Method. This helps to search for actions associated with higher Q-values in a broad neighborhood and is robust to local noise in the Q function approximation.  
Taken together, these components define \emph{GRAC}, a novel self-guided and self-regularized actor critic algorithm. We evaluate our method on the suite of OpenAI gym tasks, achieving or outperforming state of the art in every environment tested.


\bibliography{example_paper}
\bibliographystyle{plain}

\clearpage
\section{Appendix - GRAC: Self-Guided and Self-Regularized Actor-Critic}
\subsection{Implementation Details}
\subsubsection{Neural Network Architecture Details}
We use a two layer feedforward neural network of 256 and 256
hidden nodes respectively. Rectified linear units~(ReLU) are put after each layer for both the actor and critic except the last layer. For the last layer of the actor network, a tanh function is used as the activation function to squash the action range within $[-1,1]$. \emph{GRAC} then multiplies the output of the tanh function by \emph{max action} to transform [-1,1] into [-\emph{max action}, \emph{max action}]. The actor network outputs the mean and sigma of a Gaussian distribution.

\subsubsection{CEM Implementation}
Our CEM implementation is based on the CEM algorithm described in Pourchot~\etal~\cite{pourchot2018cemrl}.

\begin{algorithm}[H]
{Input: Q-function Q(s,a);  size of population $N_{pop}$; size of elite $N_{elite}$ where $N_{elite} \leq N_{pop}$; max iteration of CEM $N_{cem}$.}\\
{Initialize the mean $\mu$ and covariance matrix $\Sigma$ from actor network predictions.}
\begin{algorithmic}[1]
	\For {$i=1...,N_{\text{cem}}$}
		\State Draw the current population set $\{a_{pop} \}$ of size $N_{pop}$ from $\mathcal{N}(\mu, \Sigma)$.
		\State Receive the $Q$ values $\{q_{pop}\} = \{Q(s,a) | a \in \{a_{pop}\} \}$.
	    \State Sort $\{q_{pop}\}$ in descending order.
 	    \State Select top $N_{elite}$ Q values and choose their corresponding $a_{pop}$ as elite $\{a_{elite}\}$. 
	    \State Calculate the mean $\mu$ and covariance matrix $\Sigma$ of the set  $\{ a_{elite} \}$.
 	\EndFor
\end{algorithmic} 
{Output: The top one elite in the final iteration.}
\caption{CEM}
\label{alg:cem}
\end{algorithm}

\subsubsection{Additional Detail on Algorithm 1: GRAC}
The actor network outputs the mean and sigma of a Gaussian distribution. In Line 2 of Alg.1, the actor has to select action $a$ based on the state $s$. In the test stage, the actor directly uses the predicted mean as output action $a$. In the training stage, the actor first samples an action $\hat{a}$ from the predicted Gaussian distribution $\pi_{\phi}(s)$, then \emph{GRAC} runs \emph{CEM} to find a second action $\tilde{a}=\emph{CEM}(Q(s,\cdot;\theta_2),\pi_{\phi}(\cdot | s)$). \emph{GRAC} uses 
 $a=\argmax_{\{\tilde{a},\hat{a}\}}\{\min_{j=1,2} Q(s,\tilde{a};\theta_j), \min_{j=1,2} Q(s,\hat{a};\theta_j)\}$ as the final action.
 
\subsection{Appendix on Experiments}
\subsubsection{Hyperparameters used}\label{appendsubsec:hyperparams}
Table~\ref{tab:hyper} and Table~\ref{tab:hyper_env} list the hyperparameters used in the experiments. $[a,b]$ denotes a linear schedule from $a$ to $b$ during the training process. 
\begin{table}[H]
\centering
\begin{tabular}{@{}p{0.3\linewidth}p{0.15\linewidth}}
\hline \hline 
Parameters & Values  \\ \hline
\hline
discount $\gamma$ & 0.99\\
\hline
replay buffer size & 1e6\\
\hline
batch size & 256 \\
\hline
optimizer & Adam~\cite{kingma2014adam}\\
\hline 
learning rate in critic & 3e-4\\
\hline 
learning rate in actor & 2e-4\\
\hline
$N_{\text{cem}}$ & 2\\
\hline
$N_{\text{pop}}$ & 256\\
\hline
$N_{\text{elite}}$ & 5\\
\hline\hline
\end{tabular}
\caption{Hyperparameter Table}
\label{tab:hyper}
\end{table}

\begin{table}[H]
\centering
\begin{tabular}{@{}p{0.17\linewidth}p{0.12\linewidth} p{0.11\linewidth} p{0.11\linewidth} p{0.15\linewidth} p{0.14\linewidth}}

\hline \hline 
Environments & ActionDim & $K$ in Alg.1 & $\alpha$ in Alg.1 & CemLossWeight & Reward Scale \\ \hline
Ant-v2 &  8 & 20 & [0.7,~85] & 1.0/ActionDim & 1.0 \\
Hopper-v2 & 3 & 20 & [0.85,~0.95] & 0.3/ActionDim & 1.0\\
HalfCheetah-v2 & 6 & 50 & [0.7,~0.85] & 1.0/ActionDim & 0.5\\
Humanoid-v2 & 17 & 20 & [0.7,~0.85] & 1.0/ActionDim & 1.0\\
Swimmer-v2 & 2 & 20 & [0.5,~0.75] & 1.0/ActionDim & 1.0\\
Walker2d-v2 & 6 & 20 & [0.8,~0.9] & 0.3//ActionDim & 1.0\\
\hline\hline
\end{tabular}
\caption{Environment Specific Parameters}
\label{tab:hyper_env}
\end{table}

\subsubsection{Additional Learning Curves for Policy Improvement with Evolution Strategy}
\label{appendsubsec:ablation_evolutionary}

\begin{figure}[H]
\centering

\begin{tabular}{ccc}
    \begin{minipage}{.3\textwidth}
    \includegraphics[width=.85\linewidth]{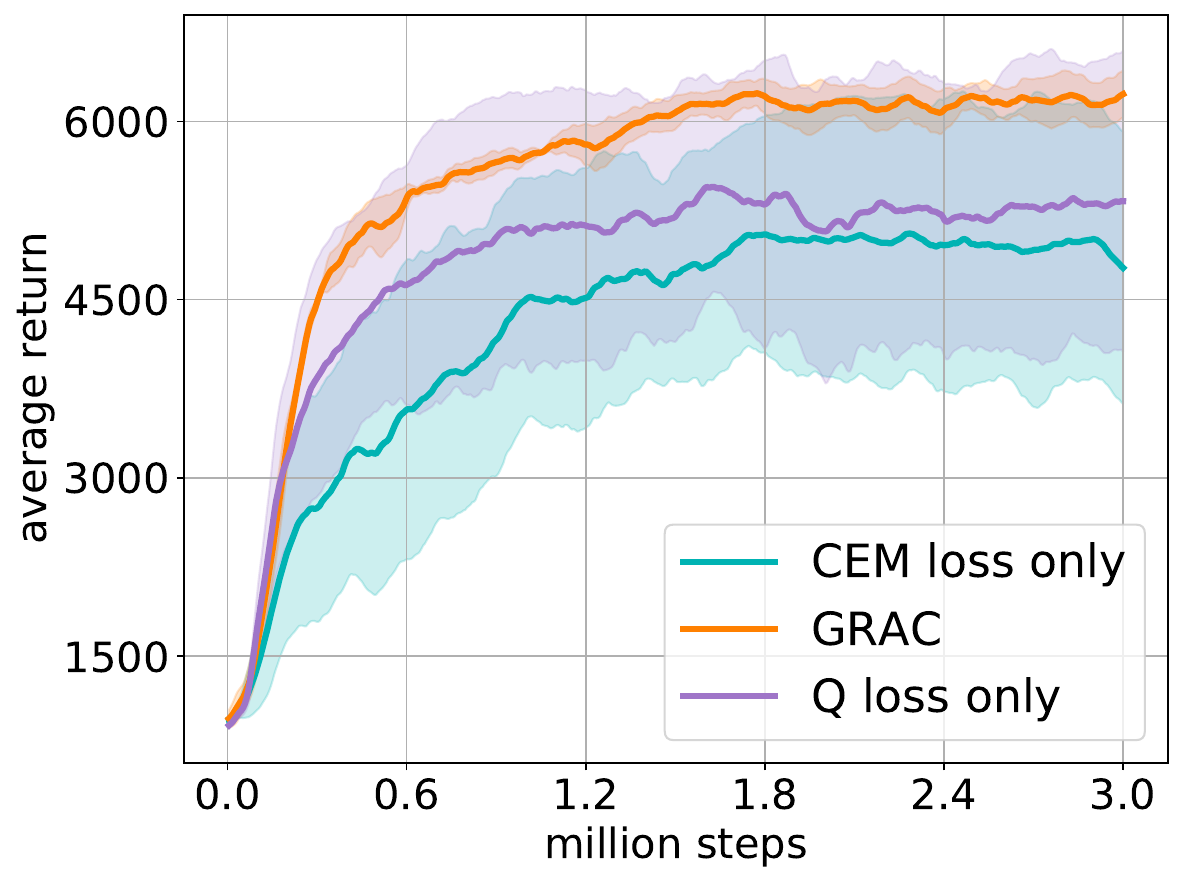}
    \centering
    \par\small{(a)~Returns on Ant-v2}
    \end{minipage}
&
    \begin{minipage}{.3\linewidth}
    \centering
        \includegraphics[width=.85\linewidth]{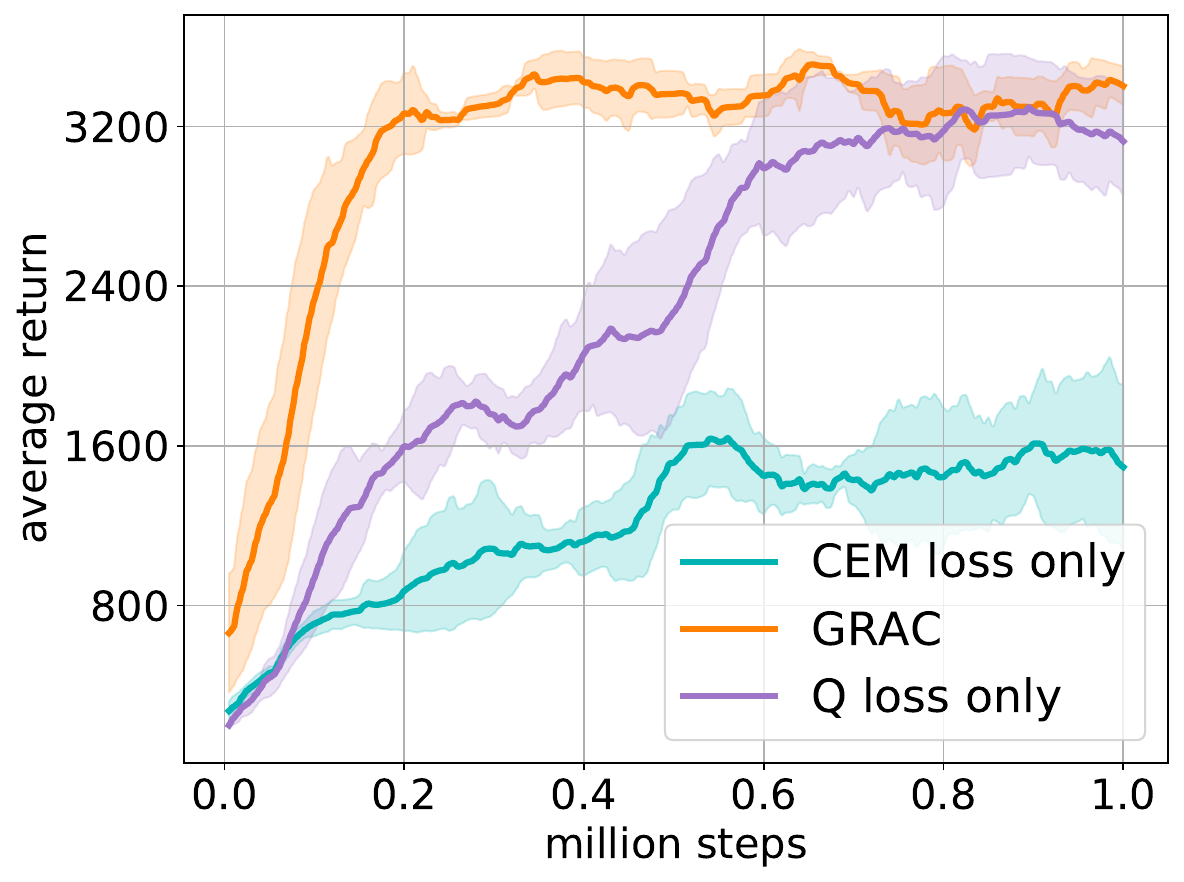}
    \par\small{(b)~Returns on Hopper-v2}
   \end{minipage}
& 
    \begin{minipage}{.3\linewidth}
    \centering
        \includegraphics[width=.85\linewidth]{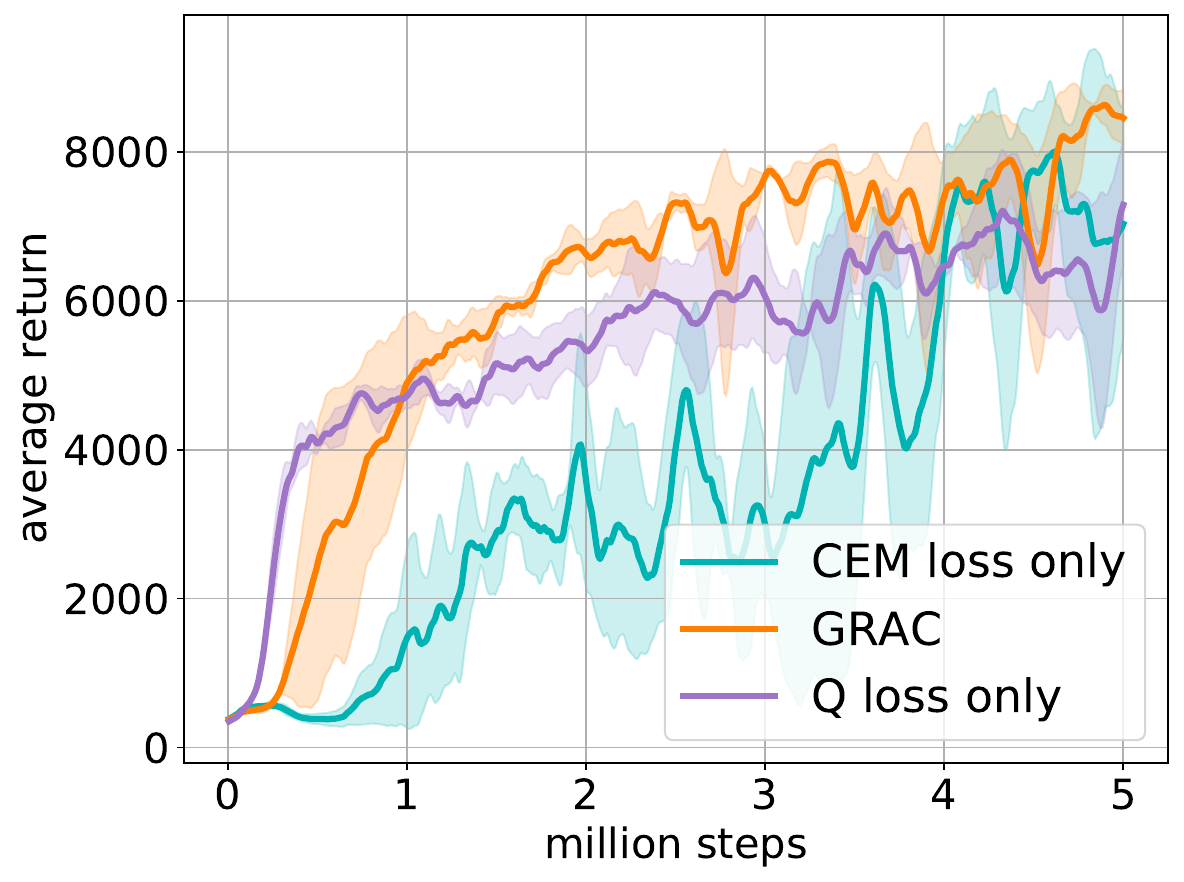}
    \par\small{(c)~Returns on Humanoid-v2}
    \end{minipage}
\end{tabular}

\begin{tabular}{ccc}
    \begin{minipage}{.3\linewidth}
    \centering
        \includegraphics[width=.85\linewidth]{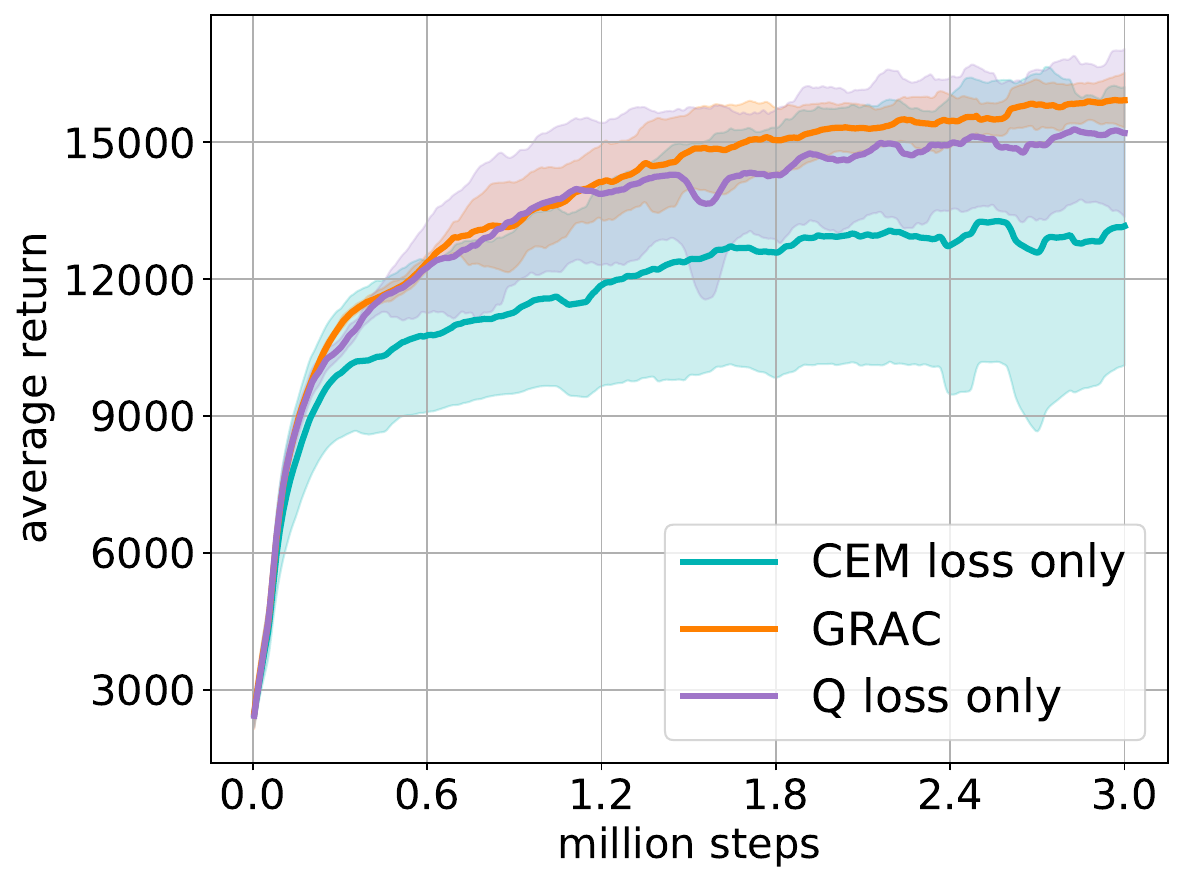}
    \par\small{(d)~Returns on HalfCheetah-v2}
    \end{minipage}
&
 	\begin{minipage}{.3\linewidth}
    \centering
        \includegraphics[width=.85\linewidth]{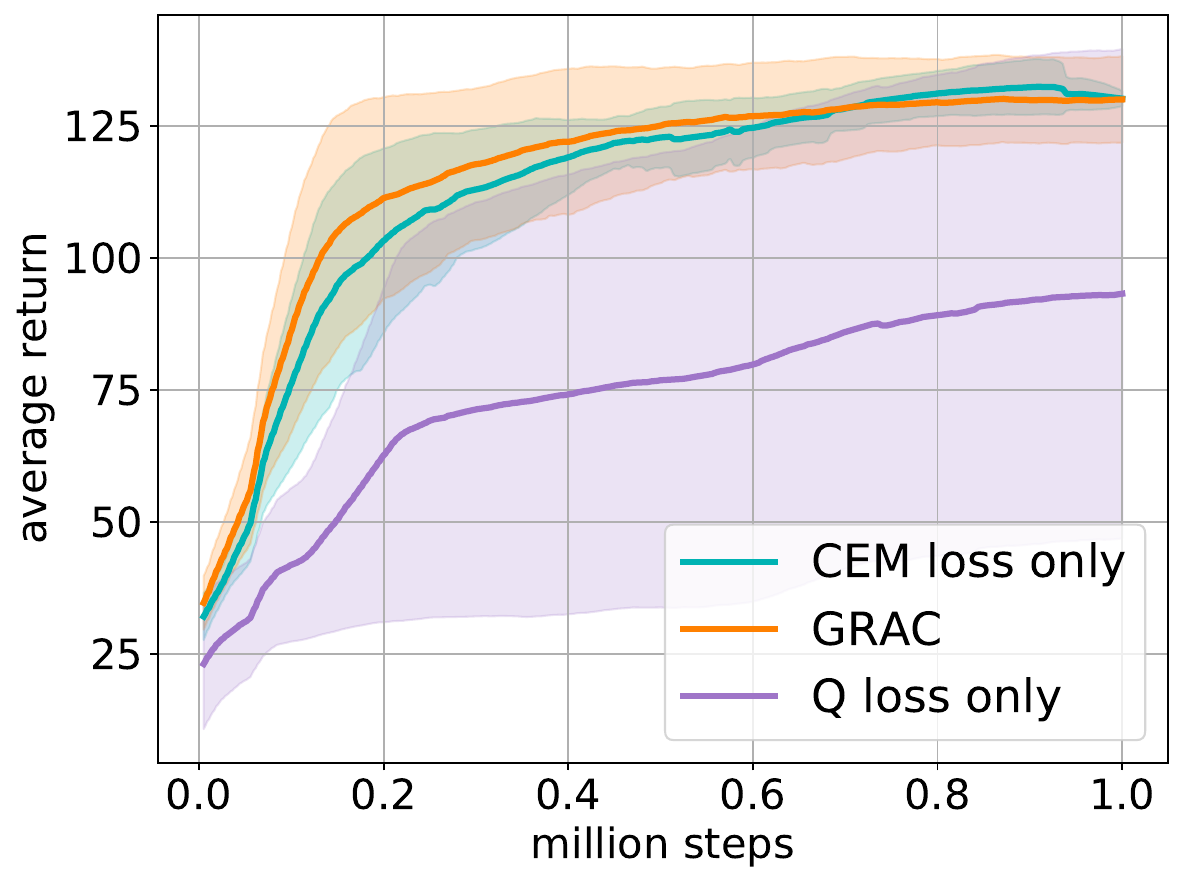}
    \par\small{(e)~Returns on Swimmer-v2}
   \end{minipage}
& 
    \begin{minipage}{.3\linewidth}
    \centering
        \includegraphics[width=.85\linewidth]{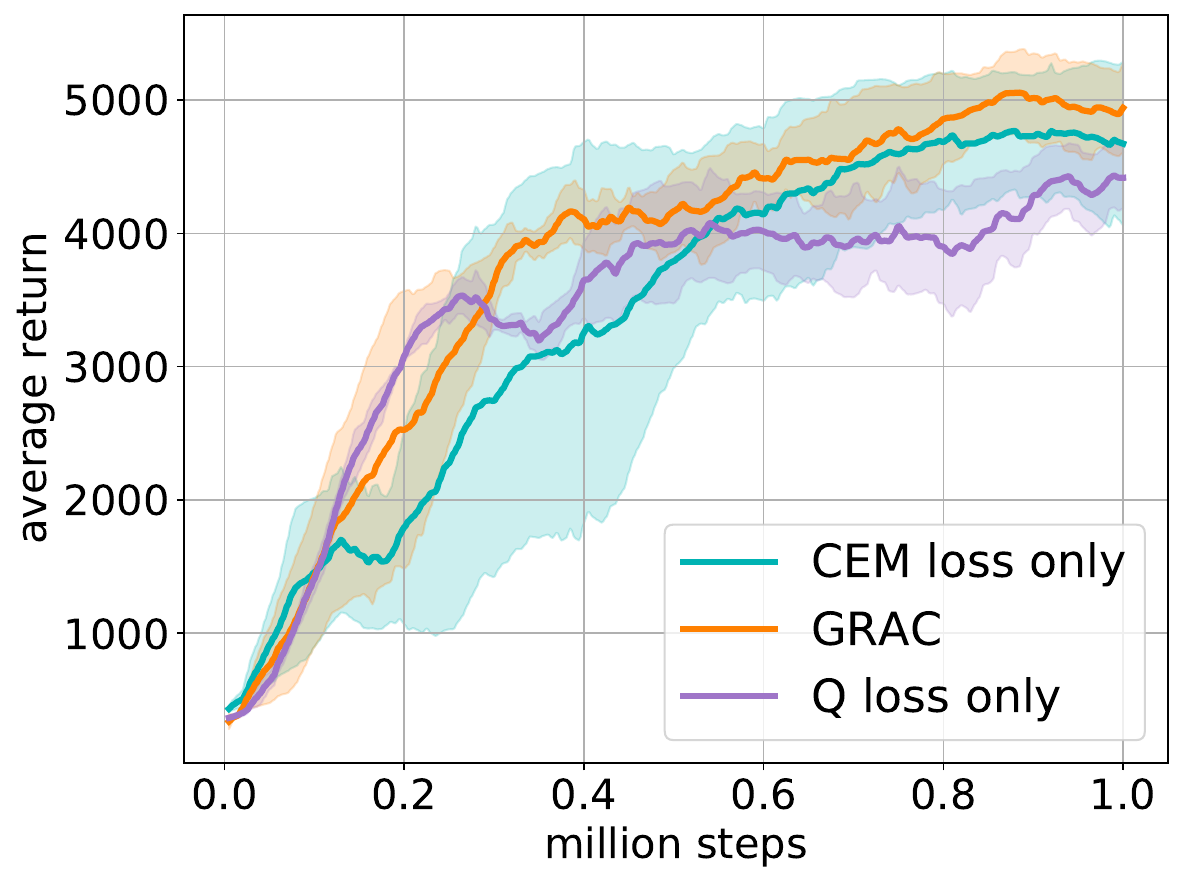}
    \par\small{(f)~Returns on Walker2d-v2}
    \end{minipage}\\ 
\end{tabular}
\caption{Learning curves for the OpenAI gym continuous control tasks. The \emph{GRAC} actor network uses a combination of two actor loss functions, denoted as QLoss and CEMLoss. \emph{QLoss Only} represents the actor network only trained with QLoss. \emph{CEM Loss Only} represents the actor network only trained with CEMLoss. In general \emph{GRAC} achieves a better performance compared to either using \emph{CEMLoss} or \emph{QLoss}.} 
\end{figure}

\subsubsection{Additional Learning Curves for Ablation Study of Self-Regularized TD Learning}
\label{appendsubsec:ablation_self_reg}
We report results in Figs~\ref{appendix:fig:AvgQ1} and ~\ref{appendix:fig:learning_crv}.

\begin{figure}
\centering
\begin{tabular}{cc}
    \begin{minipage}{.55\textwidth}
    \centering
    \includegraphics[width=\linewidth]{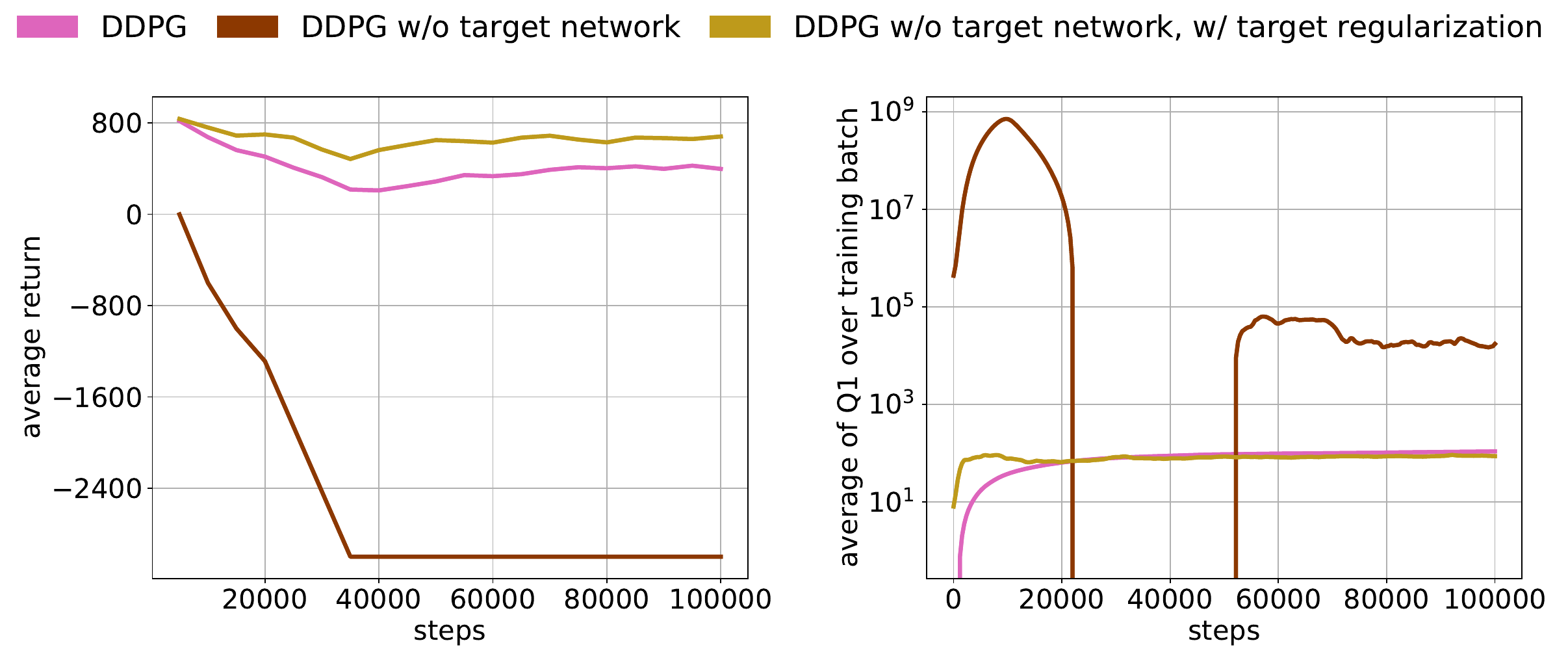}
    \end{minipage}
\end{tabular}
\begin{tabular}{cc}
    \begin{minipage}{.3\textwidth}
    \centering
    \par\small{(a)~Returns on Ant-v2}
    \end{minipage}
&
    \begin{minipage}{.3\linewidth}
    \centering
    \par\small{(b)~Average of $Q_1$ over training batch on Ant-v2}
   \end{minipage}
 \end{tabular}
\centering
\begin{tabular}{cc}
    \begin{minipage}{.55\textwidth}
    \centering
    \includegraphics[width=\linewidth]{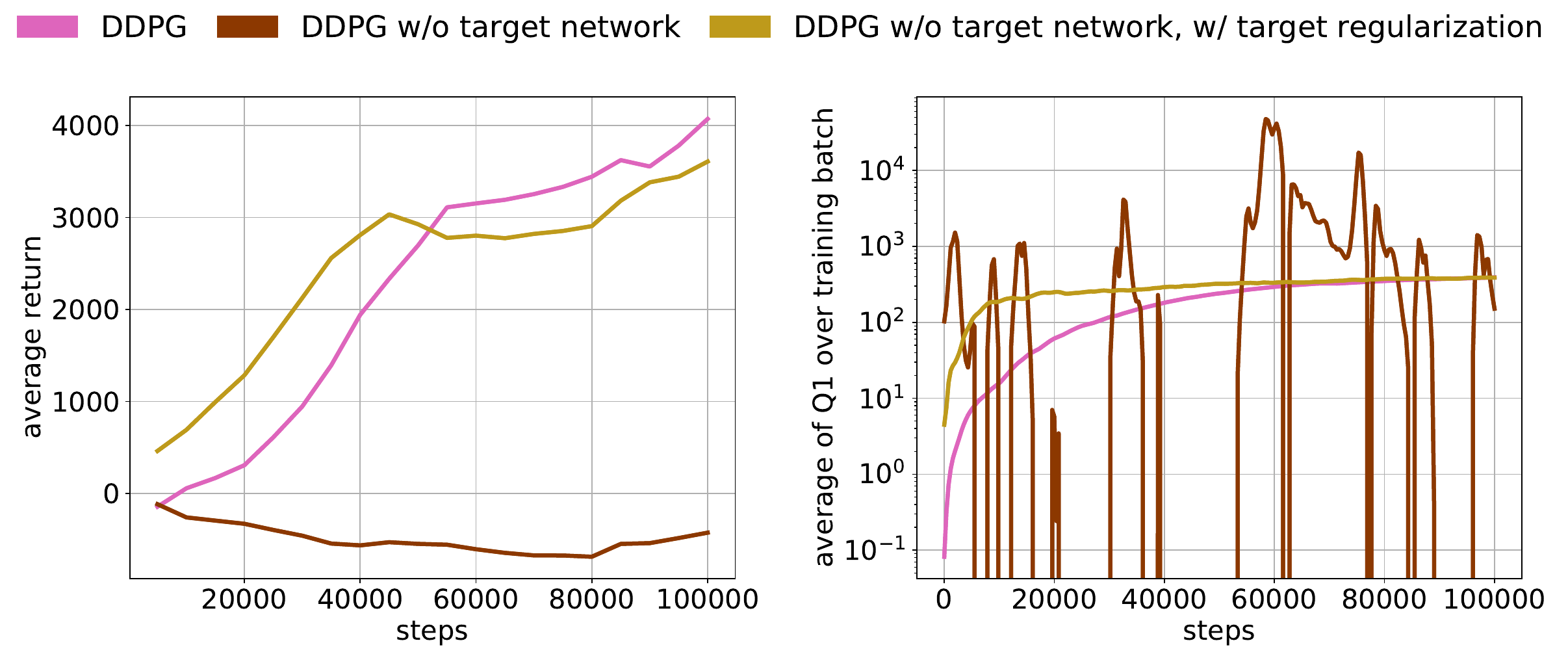}
    \end{minipage}
\end{tabular}
\begin{tabular}{cc}
    \begin{minipage}{.3\textwidth}
    \centering
    \par\small{(a)~Returns on HalfCheetah-v2}
    \end{minipage}
&
    \begin{minipage}{.3\linewidth}
    \centering
    \par\small{(b)~Average of $Q_1$ over training batch on HalfCheetah-v2}
   \end{minipage}
 \end{tabular}
\centering
\begin{tabular}{cc}
    \begin{minipage}{.55\textwidth}
    \centering
    \includegraphics[width=\linewidth]{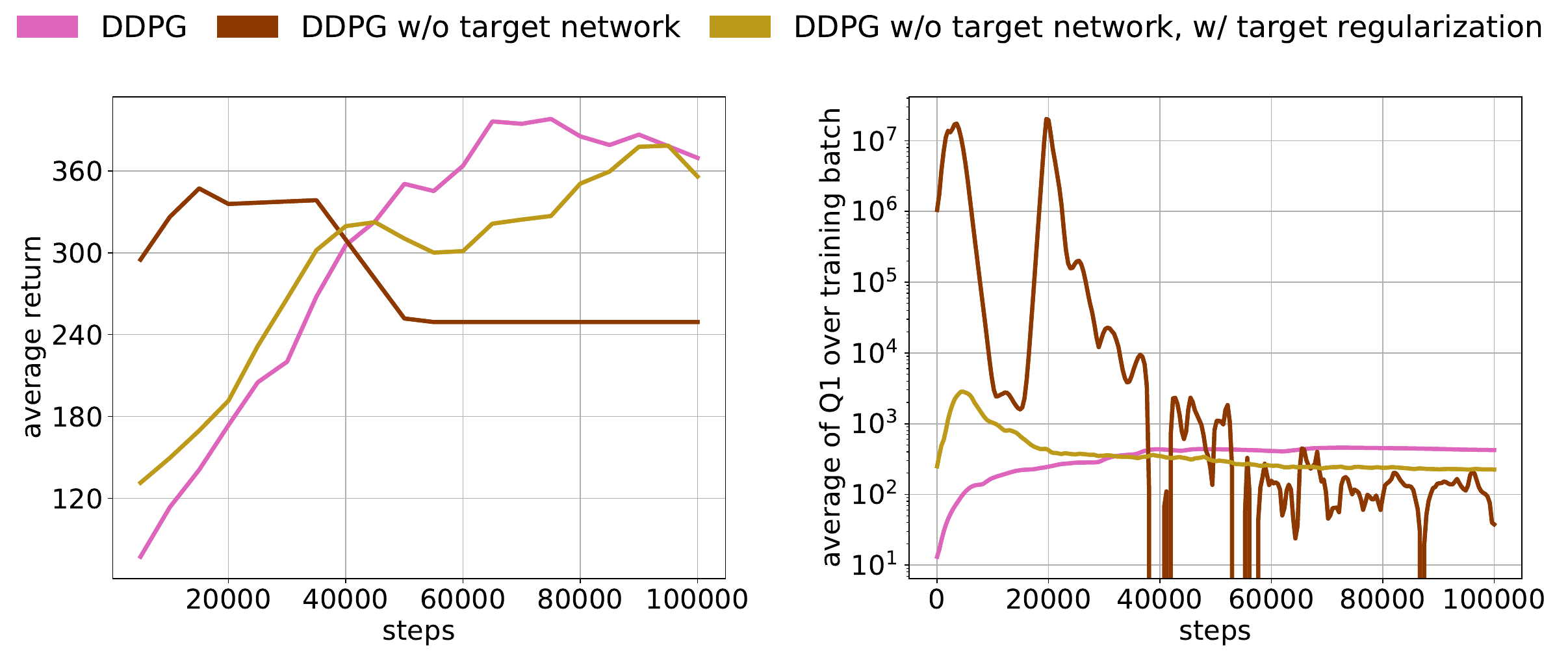}
    \end{minipage}
\end{tabular}
\begin{tabular}{cc}
    \begin{minipage}{.3\textwidth}
    \centering
    \par\small{(a)~Returns on Humanoid-v2}
    \end{minipage}
&
    \begin{minipage}{.3\linewidth}
    \centering
    \par\small{(b)~Average of $Q_1$ over training batch on Humanoid-v2}
   \end{minipage}
 \end{tabular}
\centering
\begin{tabular}{cc}
    \begin{minipage}{.55\textwidth}
    \centering
    \includegraphics[width=\linewidth]{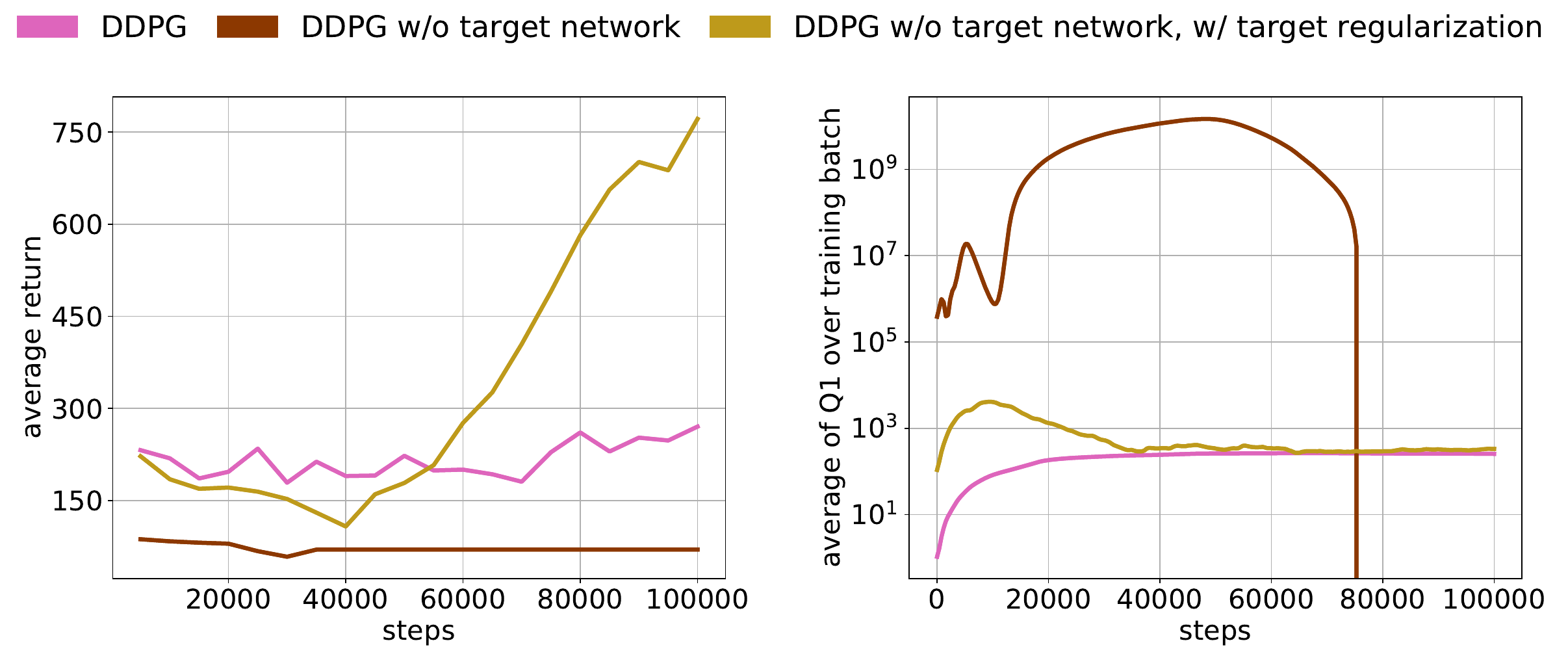}
    \end{minipage}
\end{tabular}
\begin{tabular}{cc}
    \begin{minipage}{.3\textwidth}
    \centering
    \par\small{(a)~Returns on Walker2d-v2}
    \end{minipage}
&
    \begin{minipage}{.3\linewidth}
    \centering
    \par\small{(b)~Average of $Q_1$ over training batch on Walker2d-v2}
   \end{minipage}
 \end{tabular}
\centering
\begin{tabular}{cc}
    \begin{minipage}{.55\textwidth}
    \centering
    \includegraphics[width=\linewidth]{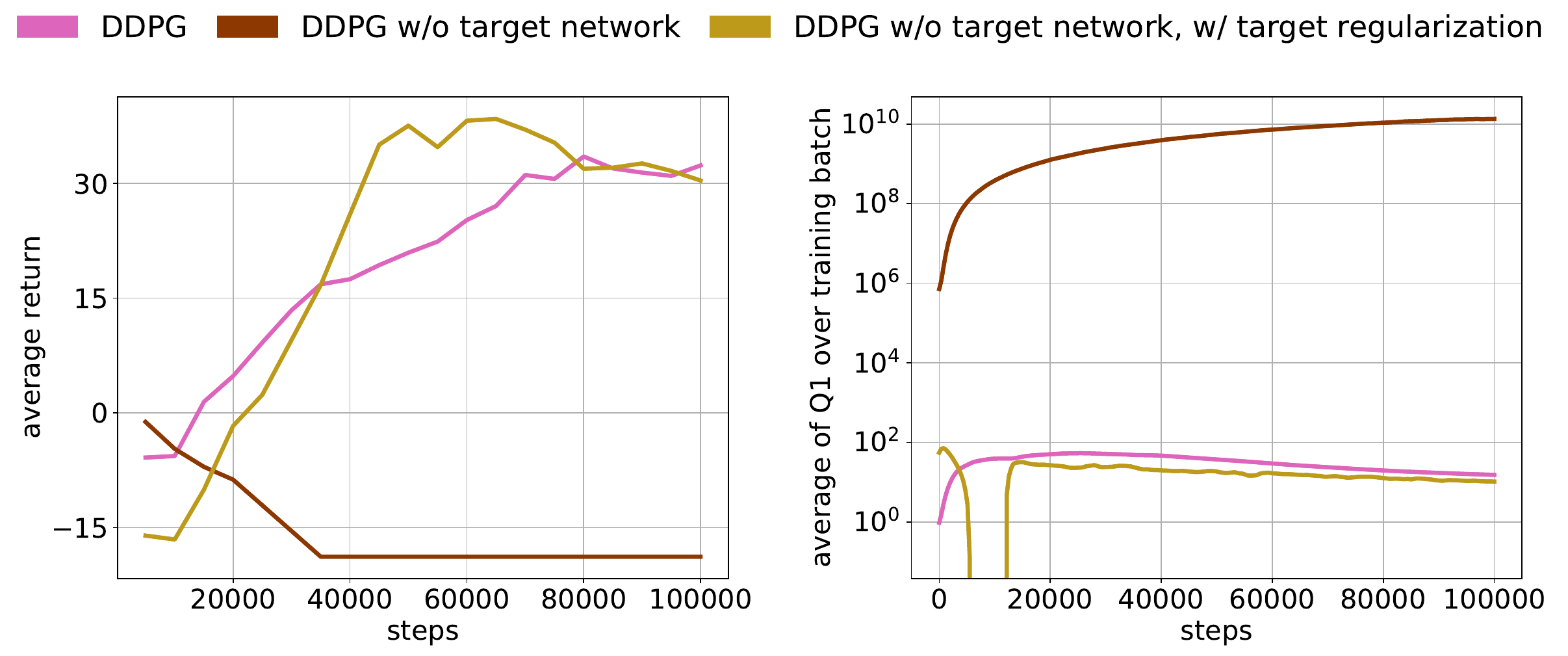}
    \end{minipage}
\end{tabular}
\begin{tabular}{cc}
    \begin{minipage}{.3\textwidth}
    \centering
    \par\small{(a)~Returns on Swimmer-v2}
    \end{minipage}
&
    \begin{minipage}{.3\linewidth}
    \centering
    \par\small{(b)~Average of $Q_1$ over training batch on Swimmer-v2}
   \end{minipage}
 \end{tabular}
 \caption{\label{appendix:fig:AvgQ1}Learning curves and average $Q_1$ values ($y^{\prime}_1$ in Alg. 1 of the main paper). \emph{DDPG} w/o target network quickly diverges as seen by the unrealistically high Q values. \emph{DDPG} is stable but often progresses slower. If we remove the target network and add the proposed target regularization, we both maintain stability and achieve a faster or comparable learning rate.}
\end{figure}


\begin{figure}[H]
\centering

\begin{tabular}{ccc}
    \begin{minipage}{.3\textwidth}
    \includegraphics[width=.85\linewidth]{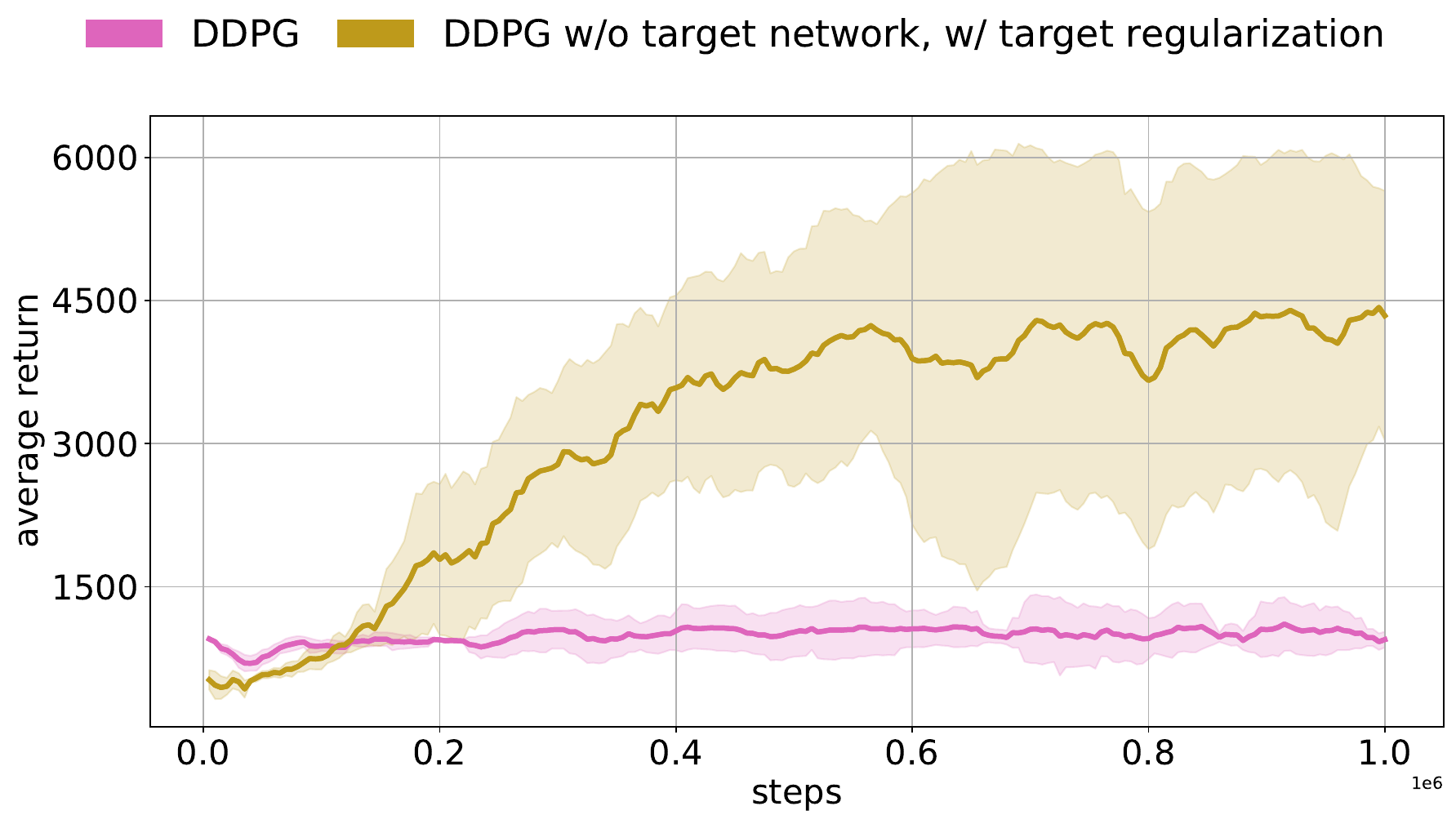}
    \centering
    \par\small{(a)~Returns on Ant-v2}
    \end{minipage}
&
    \begin{minipage}{.3\linewidth}
    \centering
        \includegraphics[width=.85\linewidth]{imgs/fig_7_hopper.pdf}
    \par\small{(b)~Returns on Hopper-v2}
   \end{minipage}
& 
    \begin{minipage}{.3\linewidth}
    \centering
        \includegraphics[width=.85\linewidth]{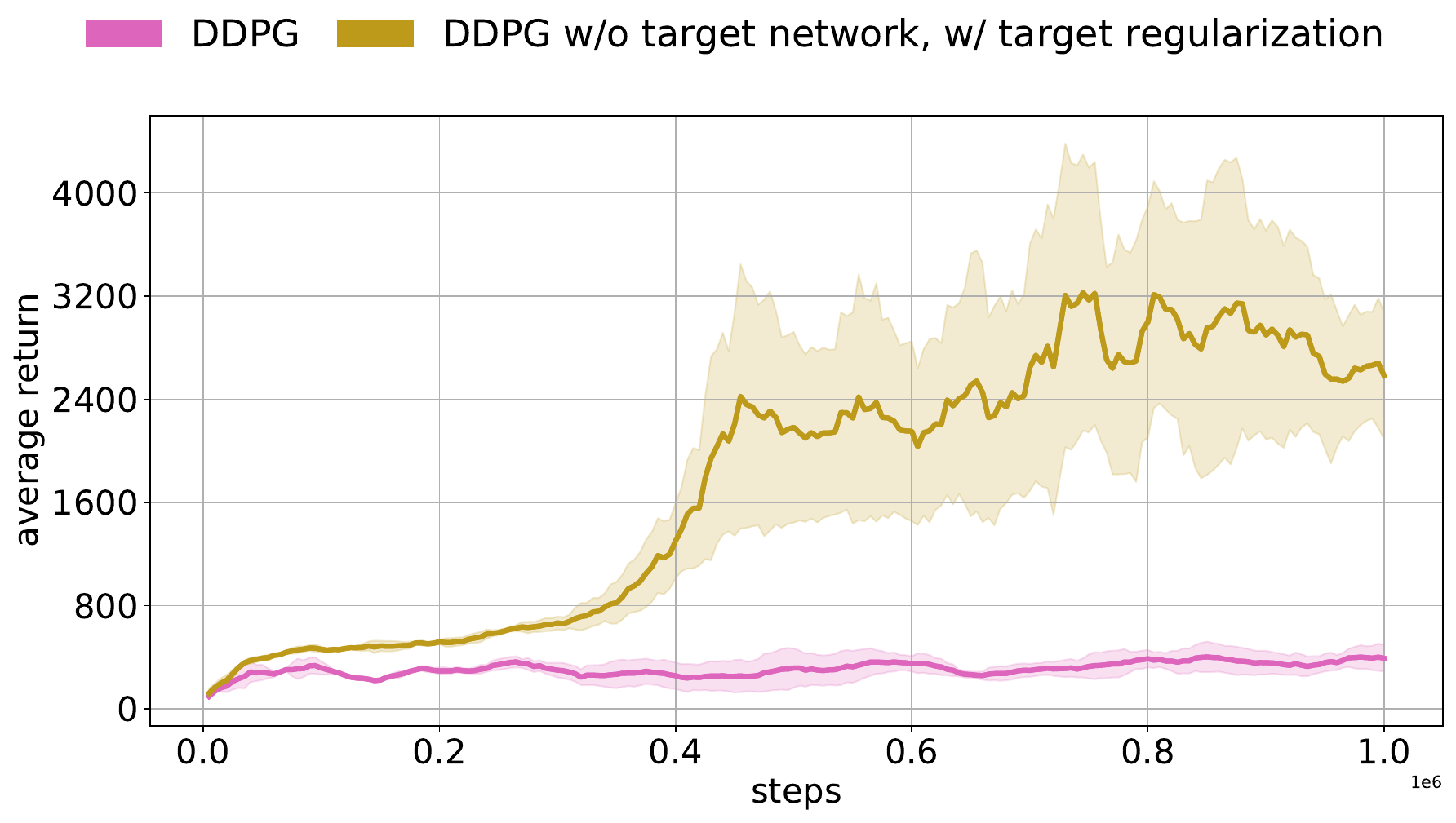}
    \par\small{(c)~Returns on Humanoid-v2}
    \end{minipage}
\end{tabular}

\begin{tabular}{ccc}
    \begin{minipage}{.3\linewidth}
    \centering
        \includegraphics[width=.85\linewidth]{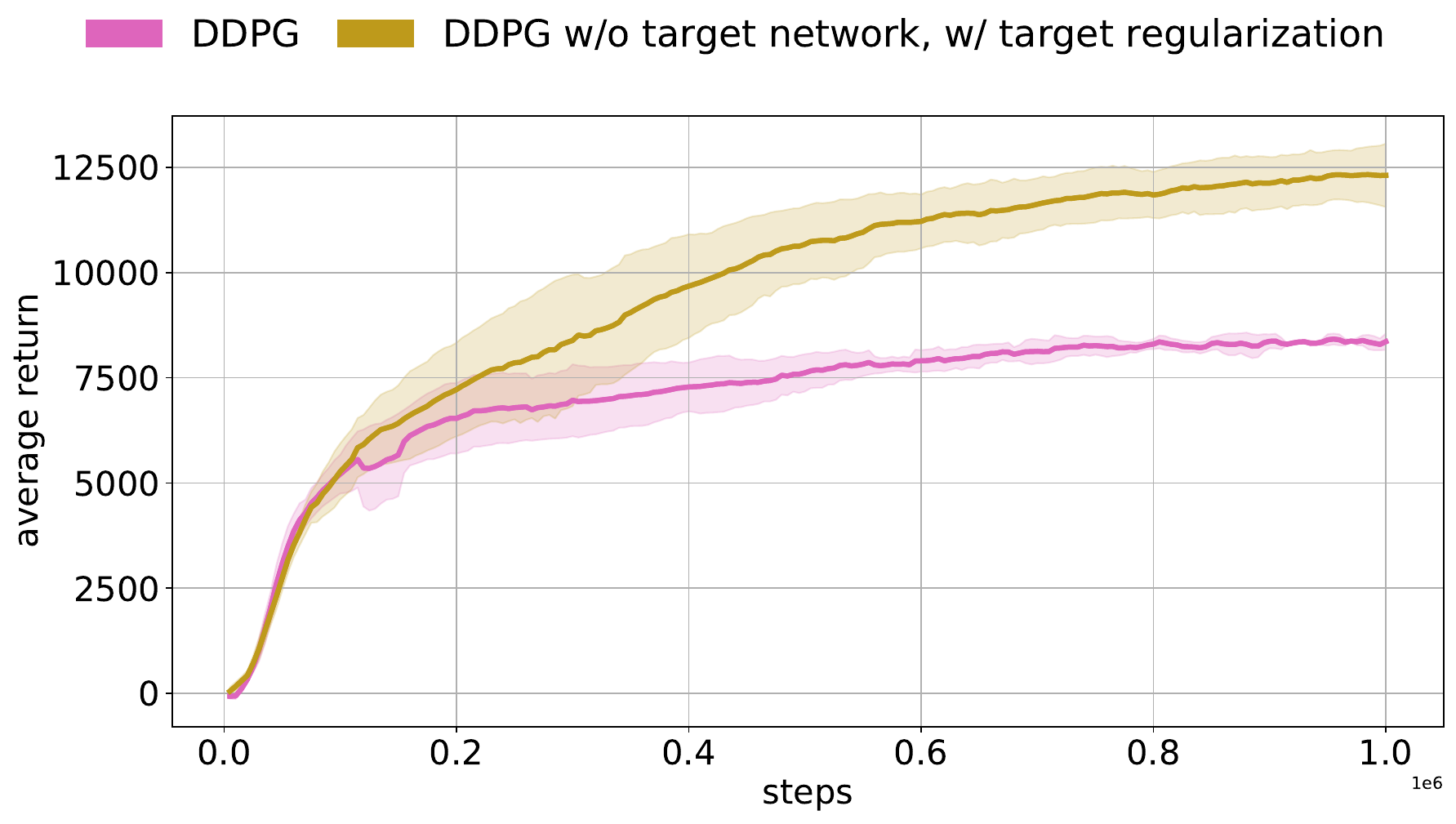}
    \par\small{(d)~Returns on HalfCheetah-v2}
    \end{minipage}
&
 	\begin{minipage}{.3\linewidth}
    \centering
        \includegraphics[width=.85\linewidth]{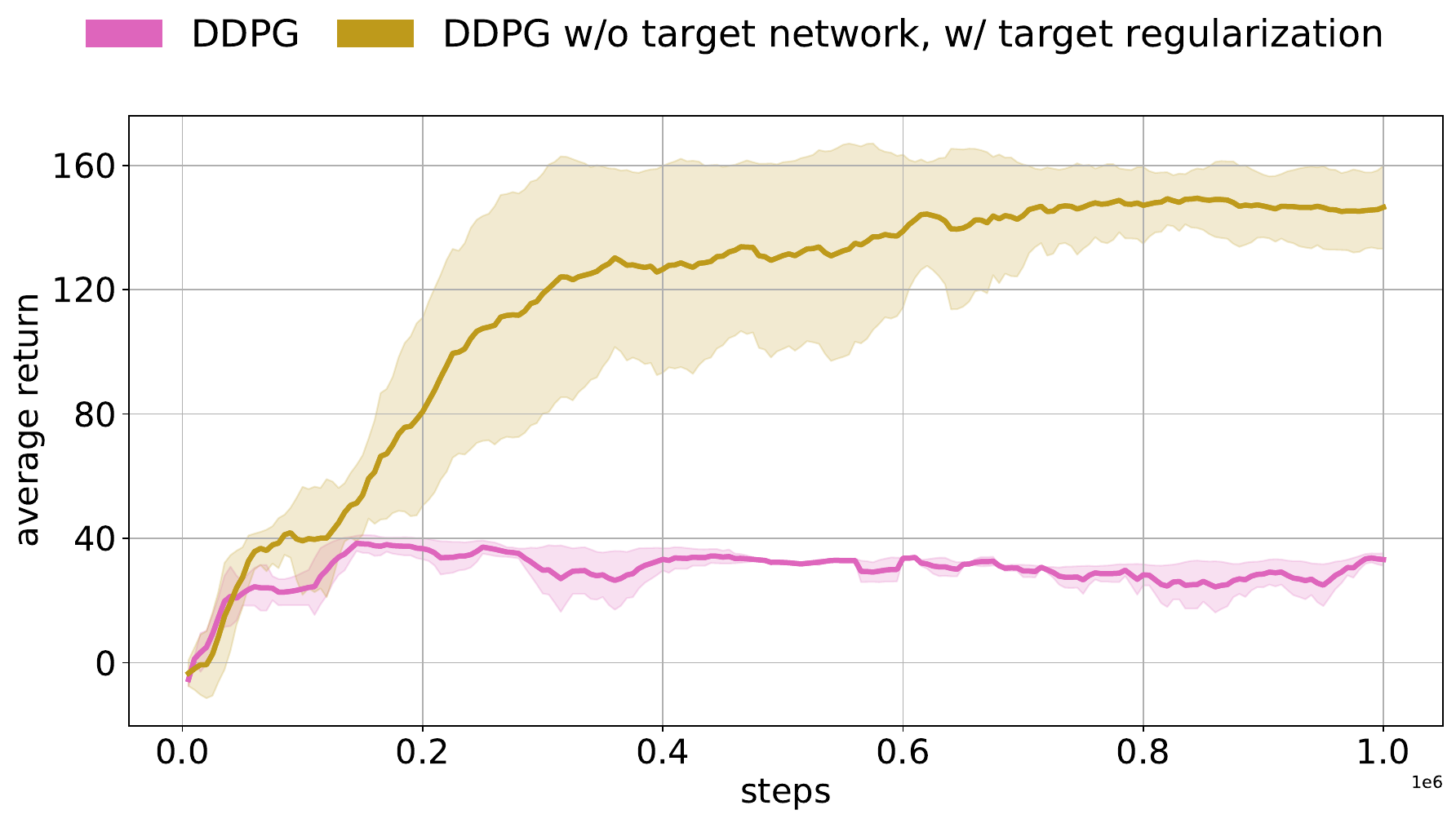}
    \par\small{(e)~Returns on Swimmer-v2}
   \end{minipage}
& 
    \begin{minipage}{.3\linewidth}
    \centering
        \includegraphics[width=.85\linewidth]{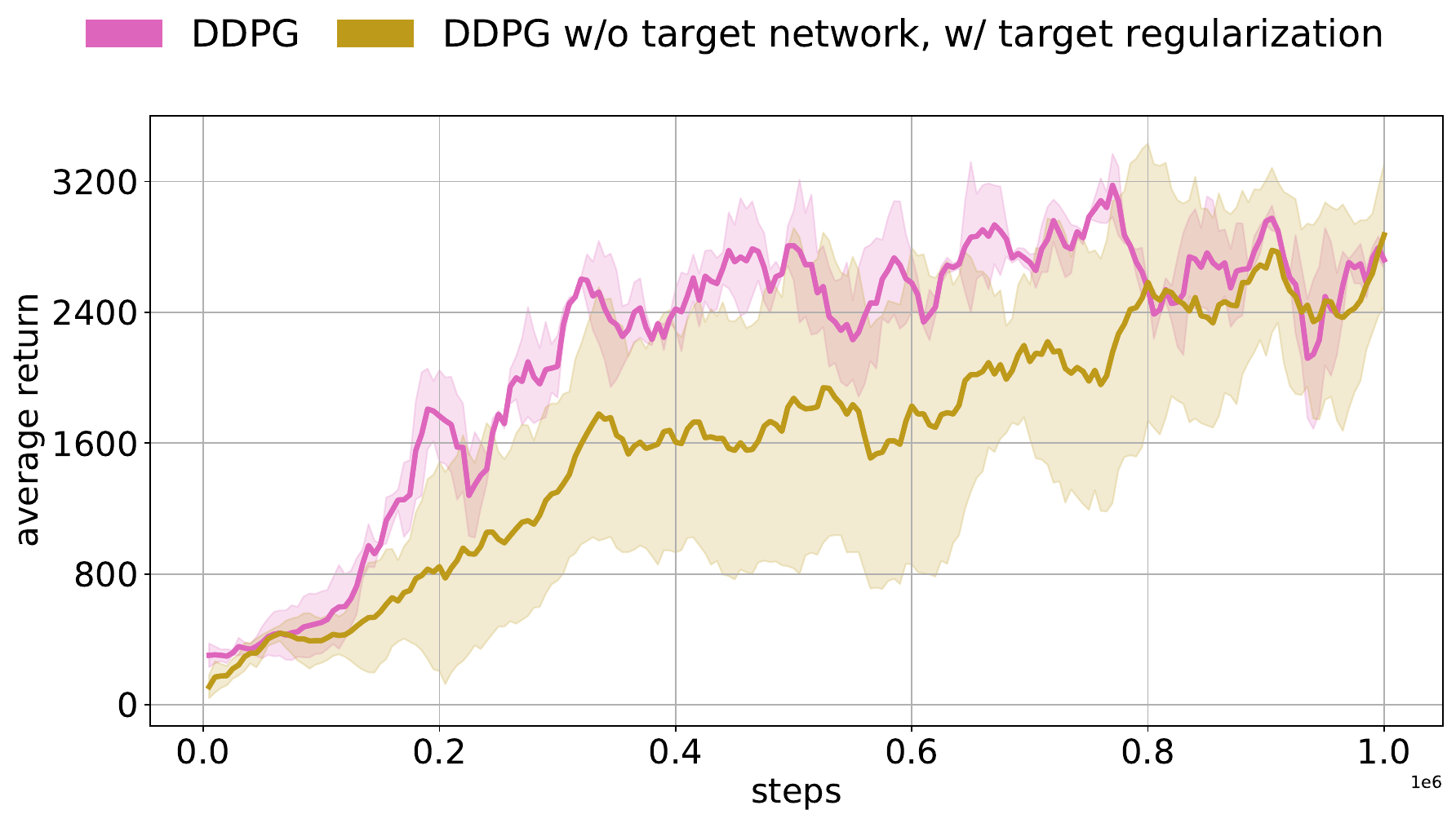}
    \par\small{(f)~Returns on Walker2d-v2}
    \end{minipage}\\ 
\end{tabular}
\caption{\label{appendix:fig:learning_crv}Learning curves of \emph{DDPG w/o target network, w/ target regularization} and \emph{DDPG} on the OpenAI gym continuous control tasks within one million steps over four random seeds. \emph{DDPG w/o target network, w/ target regularization} outperforms \emph{DDPG} by large margins in five out of six Mujoco tasks.} 
\end{figure}

\subsubsection{Hyperparameter Sensitivity for the Termination Condition of Critic Network Training}\label{appendsubsec:termination}
We also run experiments to examine how sensitive \emph{GRAC} is to some hyperparameters such as $K$ and $\alpha$ listed in Alg.1. The critic networks will be updated until the critic loss has decreased to $\alpha$ times the original loss, or at most $K$ iterations, before proceeding to update the actor network. In practice, we decrease $\alpha$ in the training process. Fig. 3 shows five learning curves on Ant-v2 running with five different hyperparameter values. We find that a moderate value of $K=10$ is enough to stabilize the training process, and increasing $K$ further does not have significant influence on training, shown on the right of Fig. 3. $\alpha$ is usually within the range of $[0.7,0.9]$ and most tasks are not sensitive to minor changes. However on the task of Swimmer-v2, we find that $\alpha$ needs to be small enough ($<0.7$) to prevent divergence. In practice, without appropriate $K$ and $\alpha$ values, divergence usually happens within the first 50k training steps, thus it is quick to select appropriate values for $K$ and $\alpha$.

\begin{figure*}[ht!]
\centering

\begin{tabular}{cc}
    \begin{minipage}{.45\textwidth}
    \includegraphics[width=.85\linewidth]{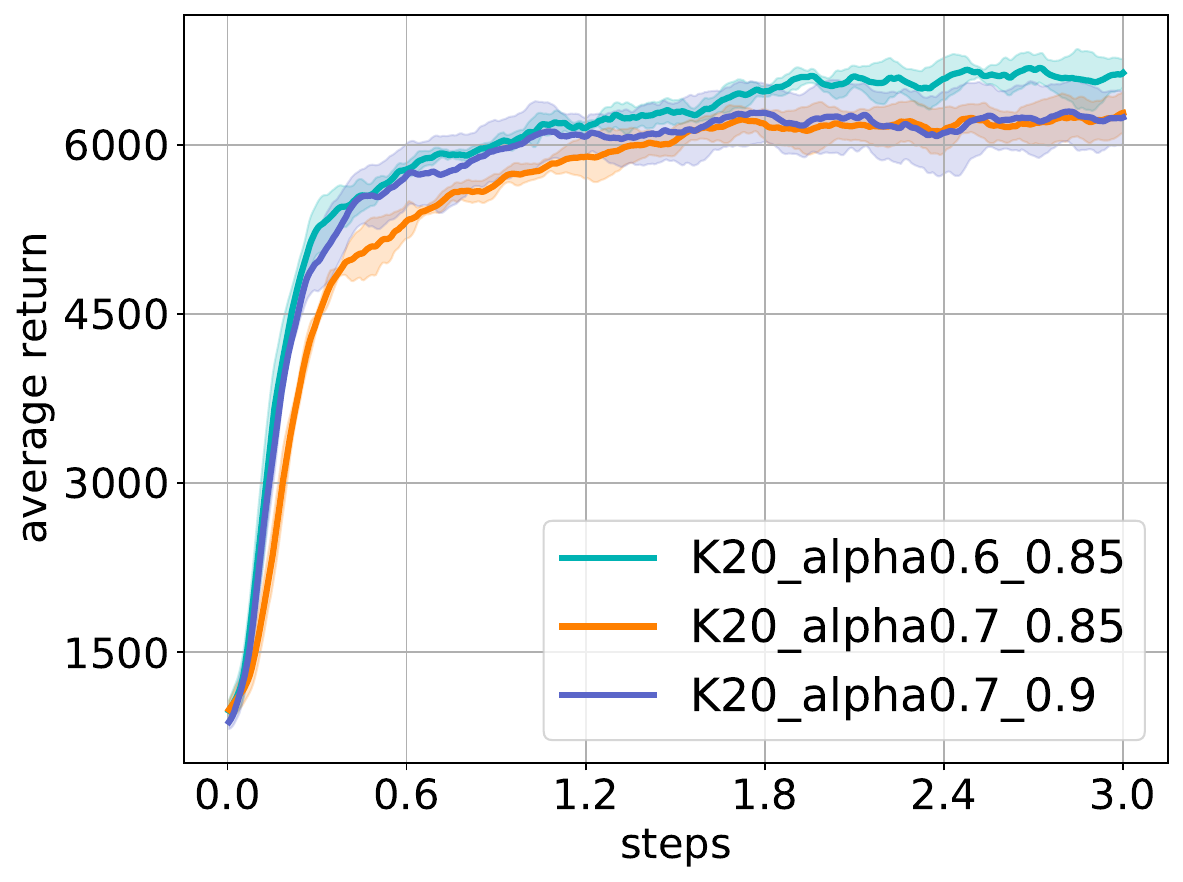}
    \centering
    \par\small{(a)~Returns on Ant-v2}
    \end{minipage}
&
    \begin{minipage}{.45\linewidth}
    \centering
        \includegraphics[width=.85\linewidth]{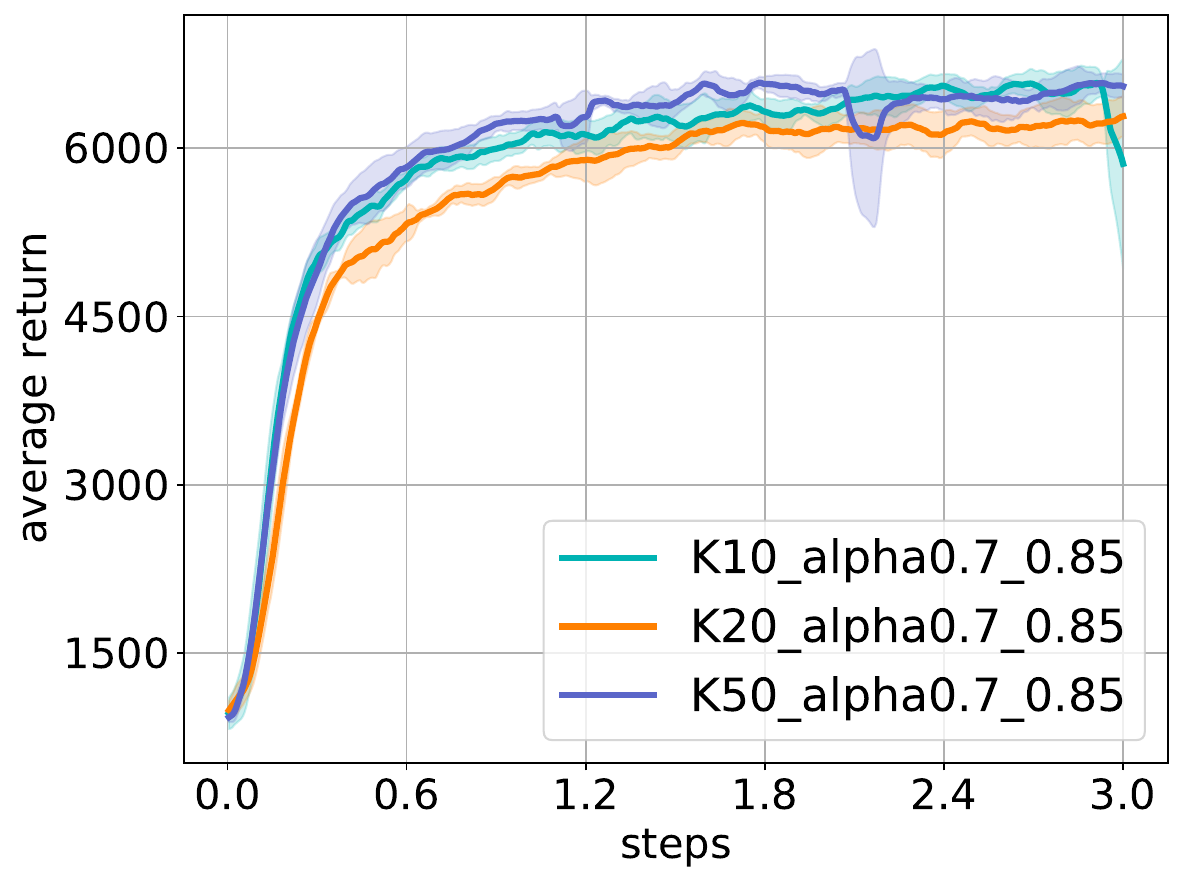}
    \par\small{(b)~Returns on Ant-v2}
   \end{minipage}
\end{tabular}
\caption{Learning curves for the OpenAI gym Ant-v2 environment.}\label{fig:results}
\end{figure*}

\subsection{Theorems and Proofs}
\label{appendsec:theorems}

For the sake of clarity, we make the following technical assumption about the function approximation capacity of neural networks that we use to approximate the action distribution. 

\textbf{State separation assumption:} 
The neural network chosen to approximate the policy family $\Pi$ is expressive enough to approximate the action distribution for each state $\pi(s,\cdot)$ separately.

\subsubsection{Theorem 1: \textbf{$Q$-loss Policy Improvement}}

\label{appendsubsec:theorem1}
\begin{theorem}
 
 Starting from the current policy $\pi$, we update the policy to maximize the objective $J_\pi = \E_{(s,a) \sim \rho_\pi(s,a)} Q^{\pi}(s,a)$. The maximization converges to a critical point denoted as $\pi_{new}$. Then the induced Q function, $Q^{\pi_{new}}$, satisfies $\forall (s,a), Q^{\pi_{new}}(s,a) \geq Q^{\pi}(s,a).$
\end{theorem}
\begin{proof}[Proof of Theorem 1]

Under the state separation assumption, the action distribution for each state, $\pi(s, \cdot)$, can be updated separately, for each state we are maximizing $\E_{a \sim \pi(s,\cdot)} Q^{\pi}(s,a)$. Therefore, we have $\forall s, \E_{a \sim \pi_{new}(s,\cdot)} Q^{\pi}(s,a) \geq \E_{a \sim \pi(s,\cdot)} Q^{\pi}(s,a) = V^{\pi}(s)$.

\begin{equation}
    \begin{array}{rl}
    Q^{\pi}(s,a) &= r(s,a) + \gamma \E_{s^{\prime}}V^{\pi}(s^{\prime}) \\
    &\leq r(s,a) + \gamma \E_{s^{\prime}} \E_{a^{\prime} \sim \pi_{new}} Q^{\pi}(s^{\prime}, a^{\prime}) \\
    &= r(s,a) + \gamma \E_{s^{\prime}} \E_{a^{\prime} \sim \pi_{new}} [r(s^{\prime}, a^{\prime}) + \gamma \E_{s^{\prime \prime}} V^{\pi}(s^{\prime \prime})] \\
    &\leq r(s,a) + \gamma \E_{s^{\prime}} \E_{a^{\prime} \sim \pi_{new}} r(s^{\prime}, a^{\prime}) + \gamma^2 \E_{s^{\prime}} \E_{a^{\prime} \sim \pi_{new}} \E_{s^{\prime \prime}} \E_{a^{\prime \prime} \sim \pi_{new}} Q^{\pi}(s^{\prime \prime}, a^{\prime \prime}) \\
    &= \ldots   \quad \mbox{(repeatedly unroll Q function )} \\
    &\leq Q^{\pi_{new}}(s,a)
    \end{array}
\end{equation}    

\end{proof}

    \subsubsection{Theorem 2: \textbf{\emph{CEM} Policy Improvement}}
    \label{appendsubsec:theorem2}
\begin{theorem}

We assume that the \emph{CEM} process is able to find the optimal action of the state-action value function, $a^*(s) = \argmax_{a}Q^{\pi}(s,a)$, where $Q^{\pi}$ is the Q function induced by the current policy $\pi$. 
By iteratively applying the update $ \E_{(s,a) \sim \rho_\pi(s,a)} [Q(s,a^*)-Q(s,a)]_{+}\nabla \log\pi(a^*|s)$, the policy converges to $\pi_{new}$. Then $Q^{\pi_{new}}$ satisfies $\forall (s,a), Q^{\pi_{new}}(s,a) \geq Q^{\pi}(s,a).$
\end{theorem}
\begin{proof}[Proof of Theorem 2]
Under the state separation assumption,  the action distribution for each state, $\pi(s, \cdot)$, can be updated separately. Then, for each state $s$, the policy $\pi_{new}$ will converge to a delta function at $a^*(s)$. Therefore we have $\forall s, \max_a Q^{\pi}(s,a) = \E_{a \sim \pi_{new}(s,\cdot)} Q^{\pi}(s,a) \geq \E_{a \sim \pi(s,\cdot)} Q^{\pi}(s,a) = V^{\pi}(s)$. Then, following Eq. (1) we have $\forall (s,a), Q^{\pi_{new}}(s,a) \geq Q^{\pi}(s,a)$
\end{proof}

    \subsubsection{Theorem 3: \textbf{Max-Min Double Q-learning Convergence}}
    \label{appendsubsec:theorem3}
\begin{theorem}

 We keep two tabular value estimates $Q_{1}$ and $Q_{2}$, and update via

\begin{equation}
    \begin{array}{rl}
        Q_{t+1, 1}(s,a) &= Q_{t,1}(s,a) + \alpha_t(s,a) (y_t-Q_{t,1}(s,a))\\
        Q_{t+1, 2}(s,a) &= Q_{t,2}(s,a) + \alpha_t(s,a) (y_t-Q_{t,2}(s,a)),
    \end{array}
\end{equation}
where $\alpha_t(s,a)$ is the learning rate and $y_t$ is the target:
\begin{equation}
    \begin{array}{cl}
        y_t & = r_t(s_t, a_t) + \gamma \max_{a' \in \{a^{\pi}, a^* \}} 
              \min_{i \in \{1, 2 \}}
              Q_{t,i}(s_{t+1}, a') \\
        a^{\pi} & \sim \pi(s_{t+1})\\
        a^*  & = argmax_{a'} Q_{t,2}(s_{t+1}, a') \\    \end{array}
\end{equation}



We assume that the MDP is finite and tabular and the variance of rewards are bounded, and $\gamma \in [0,1]$. We assume each state action pair is sampled an infinite number of times and both $Q_{1}$ and $Q_{2}$ receive an infinite number of updates. We further assume the learning rates satisfy $\alpha_t(s,a) \in [0,1]$, $\sum_t \alpha_t(s,a) = \infty$, $\sum_t [\alpha_t(s,a)]^2 < \infty$ with probability 1 and $\alpha_t(s,a)=0, \forall (s,a) \neq (s_t, a_t)$. Finally we assume \emph{CEM} is able to find the optimal action  $a^*(s) = \argmax_{a'}Q(s,a';\theta_2)$. Then Max-Min Double Q-learning will converge to the optimal value function $Q^*$ with probability $1$.
\end{theorem}
\begin{proof}[Proof of Theorem 3]
This proof will closely follow Appendix A of \cite{fujimoto2018addressing}.

We will first prove that $Q_2$ converges to the optimal Q value $Q^*$. Following notations of \cite{fujimoto2018addressing}, we have

\begin{equation*}
\begin{array}{rl}    
F_t(s_t,a_t) \triangleq& y_t(s_t,a_t) - Q^*(s_t, a_t) \\

    =& r_t + \gamma \max_{a^{\prime} \in \{a^{\pi}, a^* \}} 
              \min_{i \in \{ 1, 2 \}}
              Q_{t, i}(s_{t+1}, a^{\prime})
                - Q^*(s_t, a_t) \\
    =& F_t^Q(s_t, a_t) + c_t
\end{array}
\end{equation*}

Where 
\begin{eqnarray*}
F_t^Q(s_t, a_t) &=& r_t + \gamma Q_{t,2}(s_{t+1}, a^*) - Q^*(s_t, a_t) \\
&=& r_t + \gamma \max_{a^{\prime}} Q_{t,2}(s_{t+1}, a^{\prime}) -Q^*(s_t,a_t) \\
c_t &=& \gamma \max_{a' \in \{a^{\pi}, a^* \}} \min_{i \in \{ 1, 2 \}} Q_{t, i}(s_{t+1}, a') - \gamma Q_{t, 2}(s_{t+1}, a^*)
\end{eqnarray*}

$F^Q_t$ is associated with the optimum Bellman operator. It is well known that the optimum Bellman operator is a contractor, We need to prove $c_t$ converges to 0. 

Based on the update rules (Eq. (A2)), it is easy to prove that for any tuple $(s, a)$, $\Delta_t(s, a) = Q_{t, 1}(s,a) - Q_{t, 2}(s,a)$ converges to 0. This implies that $\Delta_t(s, a^\pi) = Q_{t,1}(s,a^\pi) - Q_{t,2}(s,a^\pi)$ converges to 0 and $\Delta_t(s, a^*) = Q_{t, 1}(s,a^*) - Q_{t, 2}(s,a^*)$ converges to 0. Therefore, $\min_{i \in \{ 1, 2 \}} Q_{t, i}(s, a) - Q_{t, 2}(s,a) \leq 0$ and the left hand side converges to zero, for $a \in {a^\pi, a^*}$.
Since we have $Q_{t, 2}(s, a^*) >= Q_{t, 2}(s, a^\pi)$, then 

\[
\min_{i \in \{1, 2\}} Q_{t, i}(s, a^*) \leq 
 \max_{a' \in \{a^{\pi}, a^* \}}  \min_{i \in \{ 1, 2 \}}
              Q_{t, i}(s, a') \leq Q_{t, 2}(s,a^*)
\]

Therefore $c_t = \gamma \max_{a' \in \{a^{\pi}, a^* \}}  \min_{ i \in \{ 1, 2 \}} Q_{t, i}(s, a') - Q_{t, 2}(s,a^*)$ converges to 0. And we proved $Q_{t, 2}$ converges to $Q^*$. 

Since for any tuple $(s, a)$, $\Delta_t(s, a) = Q_{t, 1}(s,a) - Q_{t, 2}(s,a)$ converges to 0, $Q_{t, 1}$ also converges to $Q^*$.
    \end{proof}

\subsection{Results on Atari Breakout}\label{appendsec:atari}
We apply self-regularized TD-learning method on DQN called \emph{DQN w/o target network w/ target regularization} on the Atari Breakout environment and it outperforms \emph{DQN} by 25\%. The learning curve over 50 million steps is shown in Fig.5. 
\begin{figure}[H]
\centering

\begin{tabular}{ccc}
    \begin{minipage}{.6\textwidth}
    \includegraphics[width=.85\linewidth]{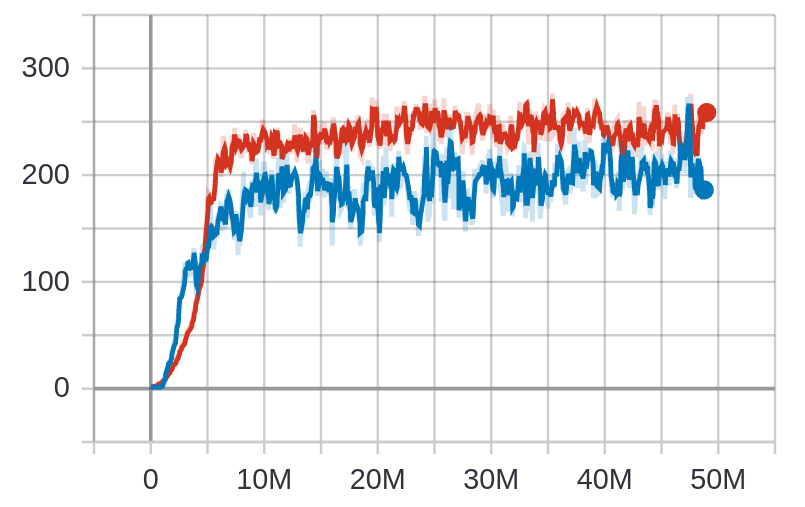}
    \centering
    \par\small{(a)~Returns on BreakoutNoFrameskip-v4}
    \end{minipage}
\end{tabular}
\caption{Average returns for the BreakoutNoFrameskip-v4 environment on OpenAI gym. Blue is \textbf{DQN}, red is \textbf{DQN w/o target network w/ target regularization.}}
\label{}
\end{figure}




\end{document}